\documentclass[10pt,twocolumn,letterpaper]{article}

\usepackage{iccv}
\usepackage{times}
\usepackage{epsfig}
\usepackage{graphicx}
\usepackage{amsmath}
\usepackage{amssymb}

\usepackage{enumitem}
\usepackage{amsfonts}
\usepackage{amsmath}
\usepackage{algorithm, algpseudocode}
\usepackage{subfigure}
\usepackage{graphicx}
\usepackage{caption}
\usepackage{booktabs} 
\usepackage[english]{babel}
\usepackage{amsthm}

\usepackage[breaklinks=true,bookmarks=false]{hyperref}

\iccvfinalcopy 

\def\iccvPaperID{****} 

\ificcvfinal\fi

\newtheorem{theorem}{Theorem}

\def\Algnameabbr{DISPEL}

\begin{document}

\title{\Algnameabbr{}: Domain Generalization via Domain-Specific Liberating}

\author{Chia-Yuan Chang\\
Texas A\&M University\\
{\tt\small cychang@tamu.edu}
\and
Yu-Neng Chuang\\
Rice University\\
{\tt\small ynchuang@rice.edu}
\and
Guanchu Wang\\
Rice University\\
{\tt\small gw22@rice.edu}
\and
Mengnan Du\\
New Jersey Institute of Technology\\
{\tt\small mengnan.du@njit.edu}
\and
Na Zou\\
Texas A\&M University\\
{\tt\small nzou1@tamu.edu}
}

\maketitle
\ificcvfinal\thispagestyle{empty}\fi

\begin{abstract}








   Domain generalization aims to learn a generalization model that can perform well on unseen test domains by only training on limited source domains. 
   However, existing domain generalization approaches often bring in prediction-irrelevant noise or require the collection of domain labels. 
   To address these challenges, we consider the domain generalization problem from a different perspective by categorizing underlying feature groups into domain-shared and domain-specific features. 
   Nevertheless, the domain-specific features are difficult to be identified and distinguished from the input data.
   In this work, we propose \underline{D}oma\underline{I}n-\underline{SPE}cific \underline{L}iberating (\Algnameabbr{}), a post-processing fine-grained masking approach that can filter out undefined and indistinguishable domain-specific features in the embedding space. Specifically, \Algnameabbr{} utilizes a mask generator that produces a unique mask for each input data to filter domain-specific features. 
   The \Algnameabbr{} framework is highly flexible to be applied to any fine-tuned models.
   We derive a generalization error bound to guarantee the generalization performance by optimizing a designed objective loss. 
   The experimental results on five benchmarks demonstrate \Algnameabbr{} outperforms existing methods and can further generalize various algorithms. 
   
\end{abstract}

\section{Introduction}
\label{sec:intro}









\begin{figure}[t!]
\centerline{\includegraphics[width=0.48\textwidth]{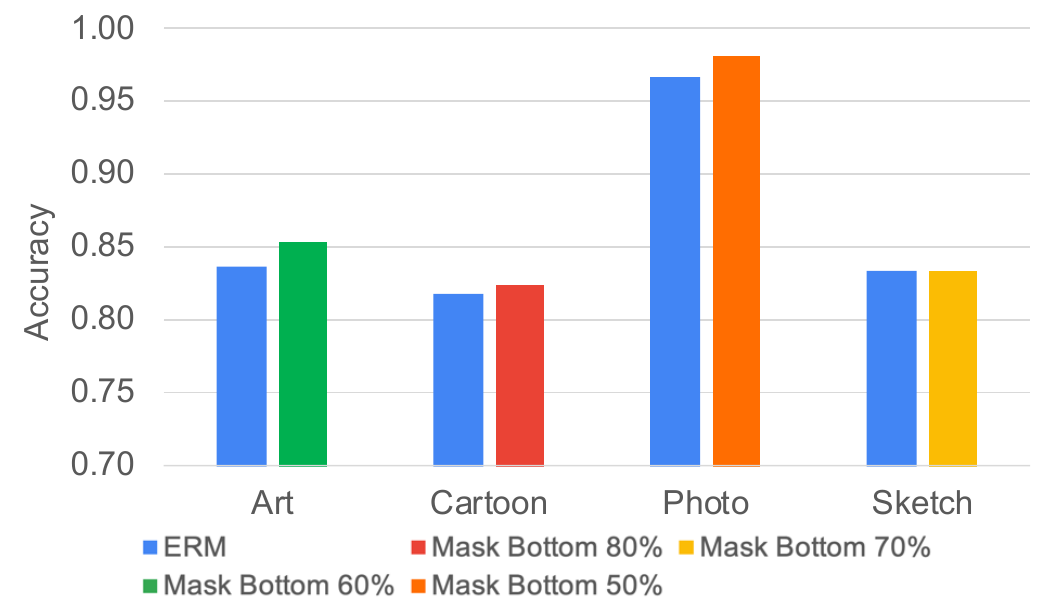}}
\caption{Efficacy of Global Masking (see Sec.~\ref{sec:global_mask}) on each unseen domain of the PACS dataset. The Mask Bottom p\% describes the percentage of masked embedding dimensions. We present the most effective p\% for each unseen domain.}
\label{fig:prelim_bar_chart}
\vspace{-0.3cm}
\end{figure}

Deep neural network (DNN) models have achieved impressive results in various fields, including object recognition~\cite{wei2019adversarial, tan2020equalization, de2019does, barbu2019objectnet, chen2019destruction}, semantic segmentation~\cite{strudel2021segmenter, cheng2021per, xie2021segformer, wang2021exploring, zheng2021rethinking}, and object detection~\cite{dai2021dynamic, joseph2021towards, tan2020efficientdet, carion2020end, chang2021multi, xie2021detco}. However, the distribution of test data in the real-world applications may be statistically different from training data, causing DNN models to severely suffer performance drops. In this case, domain generalization research aims to address the challenges of distribution difference, known as the domain shifting problem. Domain generalization is essential for real-world applications where collecting representative training data may be difficult or costly. Therefore, it is crucial to have models with good generalization properties trained on limited domain training data.

Existing algorithms for domain generalization can be categorized into \textcolor{black}{two branches}: data manipulation and representation learning.
\textcolor{black}{The first branch}, data manipulation~\cite{tobin2017domain, tremblay2018training, volpi2018generalizing, volpi2019addressing, shi2020towards, xu2020adversarial, yan2020improve, qiao2020learning}, focuses on reducing overfitting issues by increasing the diversity and quantity of available training data through data augmentation methods or generative models. Two well-known studies in this branch are Mixup~\cite{zhang2017mixup, wang2020heterogeneous} and Randomization~\cite{peng2018sim}. 
\textcolor{black}{The second branch}, representation learning~\cite{khosla2012undoing, huang2020self, li2018learning, li2018domain, ganin2016domain, li2018deep, li2018domaincidg, chattopadhyay2020learning, piratla2020efficient, cha2022domain}, aims to learn an encoder that generates  invariant embeddings among different domains. This branch includes popular algorithms such as CORAL~\cite{sun2016deep}, IRM~\cite{arjovsky2019invariant}, and DRO~\cite{sagawa2019distributionally}.
Beyond algorithms design, model selection during the training stage can also influence generalization performance. 
\textcolor{black}{A recent research}~\cite{gulrajani2020search} proposes two model selection settings in which the selected models can be trained to achieve good generalization results.
However, there are several limitations among existing approaches.
First, data manipulation methods require highly engineered efforts and can introduce too much prediction-irrelevant knowledge label noise, which will degrade prediction performance. Second, the penalty regularizations of representation learning often require domain labels or information to train domain-invariant encoder; and the existing model selection settings also need the use of domain labels to obtain desirable validation subsets. Nonetheless, obtaining these domain labels can be a significant expense or may not be feasible in real-world applications.

\begin{figure*}[t!]
\centerline{\includegraphics[width=0.9\textwidth]{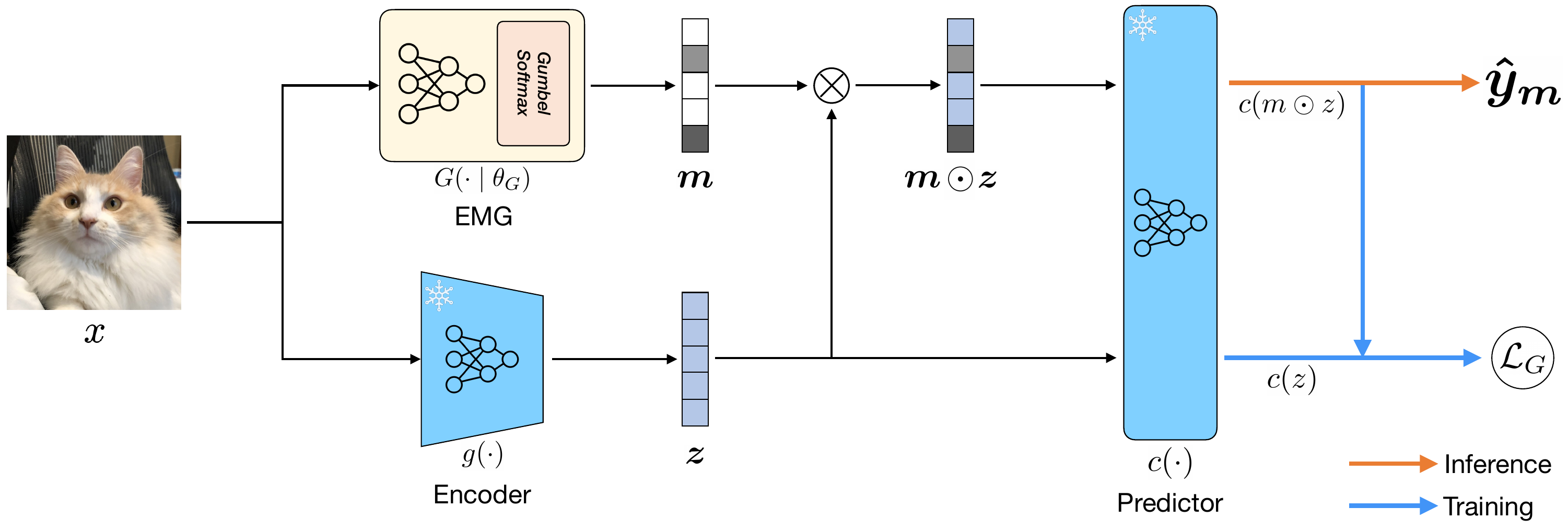}}
\caption{An overview of the proposed framework \Algnameabbr{}, where EMG 
refers to Embedding Mask Generator (see Sec.~\ref{sec:emg}), and ${\mathcal{L}}_{G}$ denotes the objective loss (see Eq.~\ref{eq:obj}).
The framework \Algnameabbr{} first splits a given frozen fine-tuned model into an encoder $g(\cdot)$ and a predictor $c(\cdot)$, and then updates EMG by minimizing cross entropy loss between $c(\boldsymbol{m} \odot \boldsymbol{z})$ and $c(\boldsymbol{z})$.}
\label{fig:model_layout}
\end{figure*}

To develop a novel solution branch without the limitations above, we consider the domain generalization problem from a different perspective by categorizing underlying feature groups, i.e., \emph{domain-shared and domain-specific features}.
The domain-shared features generally exist in all domains, and the domain-specific features only exist in certain domains.
The prediction models typically experience a drop in performance when tested on unseen domains that are excluded in training data. This is because the prediction models are trained on both domain-shared and domain-specific features, which leads to ineffective prediction outcomes on unseen domain data.
Intuitively, we can make a prediction model independent of domain-specific features by eliminating them during inference.
However, there are \textcolor{black}{two difficulties} in using this intuition.
First, domain-specific features are usually hard to be identified or predefined because those features are represented in different formats.
For instance, one domain-specific feature could be the color of objects in a source domain. This means that an object's color is highly correlated with its ground truth. 
Nevertheless, when the domain-specific features are the interweaved relationship between an object's color and stripe, it becomes difficult to recognize the correlation between each domain-specific feature and the ground truths.
Second, domain-specific features are typically difficult to be distinguished directly from the original image.
Taking an ``Art style" car image as an example. It is impalpable what is the ``Art style" car image, which only includes domain-specific features, and what is the car image without any ``style," which only contains domain-shared features.

To address the two difficulties, we perform a post-processing method, a global mask, to filter out domain-specific features in the latent space.
We verify its efficacy by designing a naive global mask to filter out the domain-specific features. The global mask is identified by ranking the Permutation Importance~\cite{fisher2019all} of each embedding dimension.
The results in Fig.~\ref{fig:prelim_bar_chart} show that utilizing the global mask can improve the generalization performance of a fine-tuned\footnote{A fine-tuned model refers to the pre-trained model that has been further trained on a specific task using a training set, where the pre-trained model is pre-trained on a large-scale dataset such as ImageNet.} model on different unseen domains.
However, it is a sub-optimal solution since the global mask does not account for the discrepancy among training instances.
As illustrated in Fig.~\ref{fig:prelim_bar_chart}, the effectiveness of masked dimensions varies for different compositions of training domains.

To account for instance discrepancies among various training domains without relying on domain labels, we propose \underline{D}oma\underline{I}n-\underline{SPE}cific \underline{L}iberating (\Algnameabbr{}), a post-processing fine-grained masking approach that extracts out domain-specific features in the embedding space.
Our proposed framework is illustrated in Fig.~\ref{fig:model_layout}.
\textbf{The key idea of \Algnameabbr{} is to learn a mask generator that automatically generates a distinct mask for each input data, which is to filter domain-specific features from the embedding space.}
The effectiveness of the proposed \Algnameabbr{} is theoretically and empirically demonstrated to improve generalization performance, without utilizing domain labels. 
The experimental results on five benchmark datasets indicate that the proposed \Algnameabbr{} achieves state-of-the-art performance, even without leveraging domain labels and any data augmentation method. This performance exceeds that of algorithms that require domain labels for training.

Our main contributions are as follows:
\begin{itemize}[leftmargin=*]
    \item We propose \Algnameabbr{} to generalize a frozen model by masking out hard-to-identify domain-specific features in the embedding space without adopting domain labels.
    \item Theoretical analysis guarantees that \Algnameabbr{} improves the generalization performance of frozen prediction models by minimizing the generalization error bound.
    \item Experimental results on various benchmarks demonstrate the effectiveness of \Algnameabbr{} and it can also further improve the existing domain generalization algorithms.
\end{itemize}

\section{A Naive Method with Global Masking}


In this section, we will go over the notations used in this paper and demonstrate the efficacy and limitation of a naive global masking strategy.

\subsection{The Problem of Domain Generalization}
A domain is a collection of data drawn from a distribution. 
Given a training domain set $\mathcal{X}$ with a label set $\mathcal{Y}$,
the goal of domain generalization is to find a perfect generalized model $f: \mathcal{X} \rightarrow \mathcal{Y}$ with parameter $\theta \in \Theta$ that can perform well on both training domains and unseen test domains.
Note that unseen test domains are not included in training domains.
Let $\mathcal{X} := \{\mathcal{D}_i\}_{i=1}^{I}$ be a set of training domains and $\mathcal{D}_i$ be a distribution over input space $\mathcal{X}$, where $I$ denotes the total number of training domains.
We denote each training domain as $\mathcal{D}_i = \{\boldsymbol{x}_{j}^{i}\}_{j=1}^\mathcal{J}$ and define a set of unseen domain samples $\mathcal{T} := \{\boldsymbol{x}_{k}^{T}\}_{k=1}^K$, where $\mathcal{J}$ denotes the total number of training domain samples and $K$ represents the total number of unseen domain samples.

\subsection{Globally Masking Domain-Specific Features} 
\label{sec:global_mask}




The performance of prediction models is typically degraded while testing on unseen domains due to accounting for domain-specific features. We experimentally show that masking out the domain-specific features in embedding space during inference can improve the generalization ability of prediction models. Specifically, given a prediction model, we split it into an encoder for mapping input data to embedding space, and a predictor for predicting. Then we leverage a feature explainer for calculating important scores for each embedding dimension, assuming that the less accumulated important scores among training data indicate the more domain-specific dimensions. Following this idea, a global mask can be found by identifying a certain number of domain-specific dimensions. Then we can leverage it to block out the domain-specific embedding dimensions to make the predictor focus on domain-shared features. In our following experiment, the global masks are obtained by calculating permutation importance~\cite{fisher2019all}. 
The results on the PACS dataset shown in Fig.~\ref{fig:prelim_bar_chart}, demonstrating that prediction accuracy on unseen domains of a fine-tuned ERM model can be improved by a global mask blocking out a certain percent of dimensions.

Despite of the efficacy of the global masking method, it is a sub-optimal solution because a global mask cannot consider the discrepancy among instances of different domains for each sample. As shown in Fig.~\ref{fig:prelim_bar_chart}, the effectiveness of masked embedding dimensions varies among different compositions of training domains.
In other words, the improvements raised by a single global mask will vary among different unseen domains.

\section{Domain-Specific Liberating (\Algnameabbr{})} \label{sec:dispel}

To address the limitation of the global masking method as is introduced in Sec.~\ref{sec:global_mask}, we aim to achieve domain generalization by considering a fine-grained masking method.
To this end, in this section, we formally propose the \underline{D}oma\underline{I}n-\underline{SPE}cific \underline{L}iberating (\Algnameabbr{}) framework that prevents a prediction model from focusing on domain-specific features projected on the dimensions of each embedding.

\subsection{Overall \Algnameabbr{} Framework} \label{sec:dispel_frame}

The prediction models being fine-tuned with ERM may have limited generalization performance on an unseen test domain. To address this, we introduce \Algnameabbr{}, a post-processing domain generalization framework that enhances the fine-tuned model's generalization performance without modifying its parameters. Specifically, \Algnameabbr{} improves its generalization by accounting for instance discrepancies among various training domains without using domain labels. We propose a fine-grained masking component named Embedding Mask Generator (EMG) that can generate instance-specific masks for blocking the domain-specific features in the embedding space (see Fig.~\ref{fig:model_layout}). The framework of \Algnameabbr{} is given as follows.
Generally, we freeze the prediction model and split it into an encoder $g(\cdot): \mathcal{X} \rightarrow \mathcal{Z}$ and a predictor $c(\cdot): \mathcal{Z} \rightarrow \mathcal{Y}$, where $\mathcal{Z}$ is the embedding space mapped by $g(\cdot)$.
Given an input instance $\boldsymbol{x} \in \mathcal{X}$, EGM can generate a mask $\boldsymbol{m}$ for embedding $\boldsymbol{z}$ of the input data to filter domain-specific features, where the embedding is encoded by the encoder $g(\boldsymbol{x}) = \boldsymbol{z}$.
Finally, the frozen model can achieve domain generalization on the unseen test domain via the mask generated by EMG.
We will introduce the implementation details in Sec.~\ref{sec:data}.
In the following, we will introduce the details and training process of the proposed EMG component.



\subsection{Embedding Mask Generator (EMG)}
\label{sec:emg}
Considering the inherent problem among training data from different domains, the Embedding Mask Generator (EMG) $G: \mathcal{X} \to \mathbb{R}^\mathrm{d}$ aims to generate an instance-specific mask to mask out domain-specific features corresponding to the input data in embedding space.
Specifically, for a data instance $\boldsymbol{x}$, a domain-specific embedding mask is generated to satisfy $d$-dimentional Binomial distribution $\mathcal{B}( 1-p_1, \cdots, 1-p_\mathrm{d})$, where $[p_1, \cdots, p_\mathrm{d}] = G(\boldsymbol{x} \mid \theta_G)$.
To enable the updating of the base model of the EMG $G(\cdot \mid \theta_G)$ through backward propagation, a temperature-dependent Gumbel Softmax operation is adopted.
Formally, the embedding mask is generated as follows:
\begin{equation}
    \boldsymbol{m}^{i}_j = \frac{\exp\{(\log (\mathbf{1} - \boldsymbol{p}) + \boldsymbol{h}) / \tau \}}{\exp\{(\log (\mathbf{1} - \boldsymbol{p}) +\boldsymbol{h}) / \tau \} + \exp\{(\log (\boldsymbol{p}) +\boldsymbol{h}') / \tau \}}
\label{eq:mask}
\end{equation}
where $\log(\cdot)$ and $\exp(\cdot)$ denote the element-wise operation of a vector; $\tau$ is a temperature hyper-parameter for controlling the degree of discreteness of the generated distribution; and $\boldsymbol{h} = [h_1, \cdots, h_d]$ and $\boldsymbol{h}' = [h'_1, \cdots, h'_d]$ are randomly sampled from $d$-dimensional standard Gumbel distribution $\text{Gumbel}(\mathbf{0}, \mathbf{1})$:
\begin{equation}
    \text{Gumbel}(\mathbf{0}, \mathbf{1}) = -\log(-\log u), u \sim \text{Uniform}(0, 1).
\label{eq:gumbel_dist}
\end{equation}

The $\tau$ in the Gumbel Softmax operation is an implicit regularization that can be used to avoid the trivial solution of EMG outputting all one mask.

\subsection{The Training of EMG Module}
\label{sec:obj_loss}
The goal of \Algnameabbr{} is to improve the generalization of a frozen prediction model by preserving domain-shared features while filtering out domain-specific features via $\boldsymbol{m}^{i}_{j} \odot \boldsymbol{z}^{i}_{j}$.
To achieve the goal, the output of frozen predictor with a masked embedding $c(\boldsymbol{m}^{i}_{j} \odot \boldsymbol{z}^{i}_{j}) = \hat{\boldsymbol{y}}_m$ should be close to the one without being masked $c(\boldsymbol{z}^{i}_{j} ) = \hat{\boldsymbol{y}}$ for all $\boldsymbol{z}^{i}_{j} \in \mathcal{Z}$ and $1 \leq j \!\leq\! \mathcal{J}$.
Following this intuition, EMG can be updated by minimizing the cross entropy loss $L_{CE}$ between $c(\boldsymbol{m}^{i}_{j} \odot \boldsymbol{z}^{i}_{j})$ and $c(\boldsymbol{z}^{i}_{j} )$ as follows:
\begin{align}
    {\mathcal{L}}_{G} 
    \notag &= \min_{{\theta}_{G}} \sum_{i=1}^{I} \sum_{j=1}^{\mathcal{J}} L_{CE} \Big( c(\boldsymbol{m}^{i}_{j} \odot \boldsymbol{z}^{i}_{j}), c(\boldsymbol{z}^{i}_{j} ) \Big) \\
    &= \min_{{\theta}_{G}} \sum_{i=1}^{I} \sum_{j=1}^{\mathcal{J}} L_{CE} \big( \hat{\boldsymbol{y}}_m, ~\hat{\boldsymbol{y}} \big).
\label{eq:obj}
\end{align}

The trained EMG $G(\cdot \mid \theta_G)$ is a domain-specific feature mask generator across multiple domains, which does not require domain labels for training.

\subsection{Generalization Bound of \Algnameabbr{}}


The objective function of EMG module in Eq.~\ref{eq:obj} aims to mitigate the performance drop on unseen test domains, which is caused by domain-specific features. The generalization error is thereby amply due to the impact of domain-specific features. In this section, we provide Theorem 1 to show that the generalization error can be effectively bounded by \Algnameabbr{}, referring to a better DG capability across the unseen domains.

\begin{theorem}[\textbf{Generalization Error Bound}] Let $\widetilde{\boldsymbol{x}}_k^{T}$ be a masked instance of $\boldsymbol{x}_k^T$ on an unseen domain $T$. Given an instance embedding $\boldsymbol{z}_k^T$ satisfies the composition of domain-specific $\boldsymbol{z}_k^{\text{T-sp}}$ and domain-sharing $\boldsymbol{z}_k^{\text{T-sh}}$, where $\hat{f}(\boldsymbol{x}_k^T) \!=\! \hat{c}(\boldsymbol{z}_k^T)$ be the predicted outcomes. For arbitrary distance function $d(\cdot,\cdot)$ in metric space, the generalization error $\textbf{GE} = \mathbf{E}_{\mathcal{X}}[ ~\textit{d}\big(f(\boldsymbol{x}_k^T), \hat{f}(\widetilde{\boldsymbol{x}}_k^{T}) \big)~ ]$ of \Algnameabbr{} framework can be bounded as:
\begin{equation} 
    \small
    \textbf{GE}
    \leq \mathbf{E}_{\mathcal{X}}[\textit{d}\big(c(\boldsymbol{z}_k^{\text{T-sh}}), \hat{c}(\widetilde{\boldsymbol{z}}_{k}^{\text{T-sh}}) \big)]
    + \mathbf{E}_{\mathcal{X}}[\textit{d}\big(c(\boldsymbol{z}_k^{\text{T-sp}}), \hat{c}(\widetilde{\boldsymbol{z}}_{k}^{\text{T-sp}}) \big)] 
\label{eq:thm_1}
\end{equation}
where $\widetilde{\boldsymbol{z}}_{k}^{\text{T}} \!=\! \boldsymbol{z}_k^{T} \odot m_k^T$ is composed of remained domain-specific embedding $\widetilde{\boldsymbol{z}}_{k}^{\text{T-sp}}$ and preserved domain-sharing embedding $\widetilde{\boldsymbol{z}}_{k}^{\text{T-sh}}$.
\label{thm:1} 
\end{theorem} 

Theorem~\ref{thm:1} shows that the upper bound of the GE depends on two terms, which are the predicted error of \textit{domain-shared features} and of \textit{domain-specific features}. The proof of Theorem 1 is provided in Appendix~\ref{appendix:proof_of_theorem1}. Based on Eq.~\ref{eq:thm_1}, \Algnameabbr{} contributes to minimizing the upper bound of GE as follows:
\begin{itemize}[leftmargin=*]
    \item The objective loss of EMG in Eq.~\ref{eq:obj} aligns with the goal of minimizing the first term of Theorem~\ref{thm:1}, as $\boldsymbol{m}^{T}_{k}$ aims to preserve the domain-shared features on $\boldsymbol{z}_{k}^{\text{T-sh}}$ for the prediction of multiple training domains.
    \item \Algnameabbr{} framework minimizes the value of second terms in Eq.~\ref{eq:thm_1} by making $\widetilde{\boldsymbol{z}}_{k}^{\text{T-sp}}$ approaches $\boldsymbol{z}_{k}^{\text{T-sp}}$ with a generated mask.
\end{itemize}

\subsection{Algorithm of EMG Training}
The training outline of the EMG is given in Algorithm~\ref{alg:train-emg}.
The training aims to achieve the base model of EMG $G(\cdot \mid \theta_G)$ that can generate domain-specific feature masks for each input data.
Specifically, EMG generates instance-specific embedding masks based on given training data (line 4), and then we let the frozen predictor $c(\cdot)$ predicts given masked embedding and original embedding (line 5-6). In each iteration, EMG is updated according to Eq.~\ref{eq:obj} (line 7) until it converges.

\begin{algorithm}
    \caption{Embedding Mask Generator (EMG) Training}
    \label{alg:train-emg}
    \small
    \begin{algorithmic}[1]
      \State {\bfseries Input:} 
      
      Training dataset $\boldsymbol{x} \in \mathcal{X}$
      
      Frozen feature encoder $g(\boldsymbol{x}) = \boldsymbol{z}$
      
      Frozen predictor $c(\cdot)$
      
      Base model of EMG $G(\cdot \mid \theta_G)$
      
      \State {\bfseries Output:} 
      
      Instance-specific mask generator EMG $G(\cdot \mid \theta_G)$
      \While{not convergence}
      \State Generate the embedding mask $\boldsymbol{m}$ by $G(\boldsymbol{x} \mid \theta_G)$ and Eq.~\ref{eq:mask}
      \State Predict by $c(\cdot)$ given masked embedding $\boldsymbol{\hat{y}_m} = c(\boldsymbol{m} \odot \boldsymbol{z})$
      \State Predict by $c(\cdot)$ given original embedding $\boldsymbol{\hat{y}} = c(\boldsymbol{z})$
      \State Update $G(\cdot \mid \theta_G)$ by minimizing the objective loss Eq.~\ref{eq:obj}
      \EndWhile
    \end{algorithmic}
\end{algorithm}



\section{Experiments}


In this section, we conduct experiments to evaluate the performance of \Algnameabbr{} framework, aiming to answer the following three research questions: 
\textbf{RQ1:} How effective is the proposed DISPEL when compared to state-of-the-art baselines?
\textbf{RQ2:} How does the fine-grained masking manner influence generalization performance?
\textbf{RQ3:} Can \Algnameabbr{} improve the generalization of other algorithms?


\begin{table}[tbh!]
    \Huge
    \renewcommand{\arraystretch}{1.0}
    \centering
    \caption{Average unseen domain results (ResNet50).}
    \makebox[0.45\textwidth][c]{
        \resizebox{0.48\textwidth}{!}{
            \begin{tabular}{l c c c c c c}
                \toprule
                 & PACS & Office-Home & VLCS & TerraInc & DomainNet & Avg. \\
                \midrule
                \multicolumn{7}{c}{\textbf{Group 1}: algorithms requiring domain labels} \\
                \midrule
                Mixup~\cite{wang2020heterogeneous} & 87.7 $\pm$ 0.5 & 71.2 $\pm$ 0.1 & 77.7 $\pm$ 0.4 & 48.9 $\pm$ 0.8 & 39.6 $\pm$ 0.1 & 65.1 \\
                MLDG~\cite{li2018learning} & 84.8 $\pm$ 0.6 & 68.2 $\pm$ 0.1 & 77.1 $\pm$ 0.4 & 46.1 $\pm$ 0.8 & 41.8 $\pm$ 0.4 & 63.6 \\
                CORAL~\cite{sun2016deep} & 86.2 $\pm$ 0.2 & 70.1 $\pm$ 0.4 & 77.7 $\pm$ 0.5 & 46.4 $\pm$ 0.8 & 41.8 $\pm$ 0.2 & 64.4 \\
                MMD~\cite{li2018domain} & 87.1 $\pm$ 0.2 & 70.4 $\pm$ 0.1 & 76.7 $\pm$ 0.9 & 49.3 $\pm$ 1.4 & 39.4 $\pm$ 0.8 & 64.6 \\
                DANN~\cite{ganin2016domain} & 86.7 $\pm$ 1.1 & 69.5 $\pm$ 0.6 & 78.7 $\pm$ 0.3 & 48.4 $\pm$ 0.5 & 38.4 $\pm$ 0.0 & 64.3 \\
                C-DANN~\cite{li2018deep} & 82.8 $\pm$ 1.5 & 65.6 $\pm$ 0.5 & 78.2 $\pm$ 0.4 & 47.6 $\pm$ 0.8 & 38.9 $\pm$ 0.1 & 62.6 \\
                DA-ERM~\cite{dubey2021adaptive} & 84.1 $\pm$ 0.5 & 67.9 $\pm$ 0.4 & 78.0 $\pm$ 0.2 & 47.3 $\pm$ 0.5 & 43.6 $\pm$ 0.3 & 64.2 \\
                \midrule
                \multicolumn{7}{c}{\textbf{Group 2}: algorithms without requiring domain labels} \\
                \midrule
                ERM~\cite{vapnik1999overview} & 86.4 $\pm$ 0.1 & 69.9 $\pm$ 0.1 & 77.4 $\pm$ 0.3 & 47.2 $\pm$ 0.4 & 41.2 $\pm$ 0.2 & 64.5 \\
                IRM~\cite{arjovsky2019invariant} & 84.4 $\pm$ 1.1 & 66.6 $\pm$ 1.0 & 78.1 $\pm$ 0.0 & 47.9 $\pm$ 0.7 & 35.7 $\pm$ 1.9 & 62.5 \\
                DRO~\cite{sagawa2019distributionally} & 86.8 $\pm$ 0.4 & 70.2 $\pm$ 0.3 & 77.2 $\pm$ 0.6 & 47.0 $\pm$ 0.3 & 33.7 $\pm$ 0.2 & 63.0 \\
                RSC~\cite{huang2020self} & 86.9 $\pm$ 0.2 & 69.4 $\pm$ 0.4 & 75.3 $\pm$ 0.5 & 45.7 $\pm$ 0.3 & 41.2 $\pm$ 1.0 & 63.7 \\
                MIRO~\cite{cha2022domain} & 85.4 $\pm$ 0.4 & 70.5 $\pm$ 0.4 & 79.0 $\pm$ 0.0 & \textbf{50.4} $\pm$ 1.1 & \textbf{44.3} $\pm$ 0.2 & 65.9 \\
                \midrule
                \Algnameabbr{} & \textbf{88.2} $\pm$ 0.1 & \textbf{73.3} $\pm$ 0.3 & \textbf{79.3} $\pm$ 0.1 & \textbf{50.4} $\pm$ 0.2 & \textbf{44.1} $\pm$ 0.0 & \textbf{67.1} \\
                \bottomrule
            \end{tabular}
        }
    }
    \label{tab:tb1}
\end{table}

\subsection{Experimental Settings} \label{sec:data}

\noindent\textbf{Datasets.}
To compare the efficacy of our proposed framework with existing algorithms, we conduct our experiments on five real-world benchmark datasets: PACS~\cite{li2017deeper} with 7 classes of images in 4 domains, Office-Home~\cite{venkateswara2017deep} with 65 classes of images in 4 domains, VLCS~\cite{fang2013unbiased} with 5 classes of images in 4 domains, Terra Incognita~\cite{beery2018recognition} with 10 classes of images in 4 domains, and DomainNet~\cite{peng2019moment} with 345 classes of images in 6 domains. DomainNet can be considered a larger-scale dataset with a more difficult multi-classification task than the other 4 benchmark datasets. More details about the datasets can be found in Appendix~\ref{appendix:datasets}.

\begin{figure*}[tbh!]
    \centering
    \subfigure[Art Painting]{
    \centering
    \begin{minipage}[t]{0.23\linewidth}
	    \includegraphics[width=0.99\linewidth]{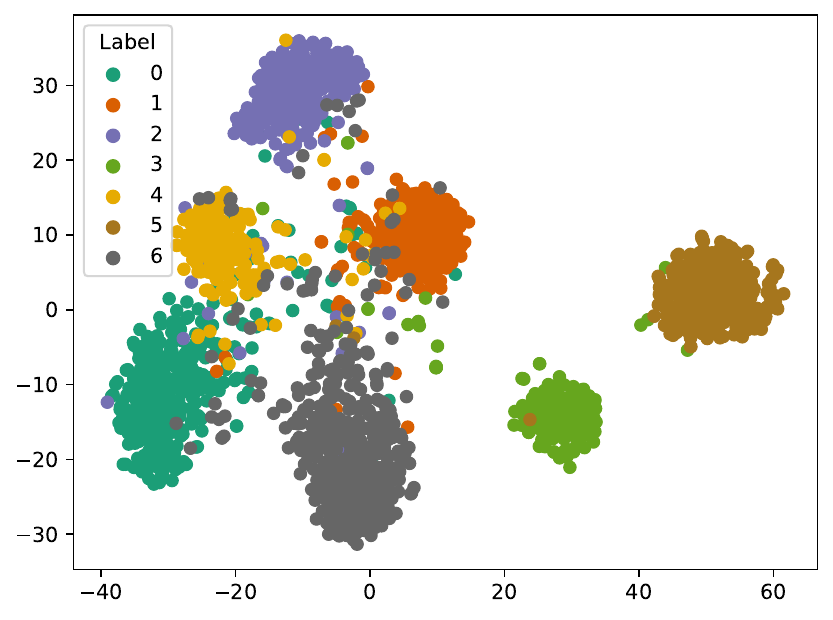}
    \end{minipage}%
    }
    \subfigure[Cartoon]{
    \centering
    \begin{minipage}[t]{0.23\linewidth}
	    \includegraphics[width=0.99\linewidth]{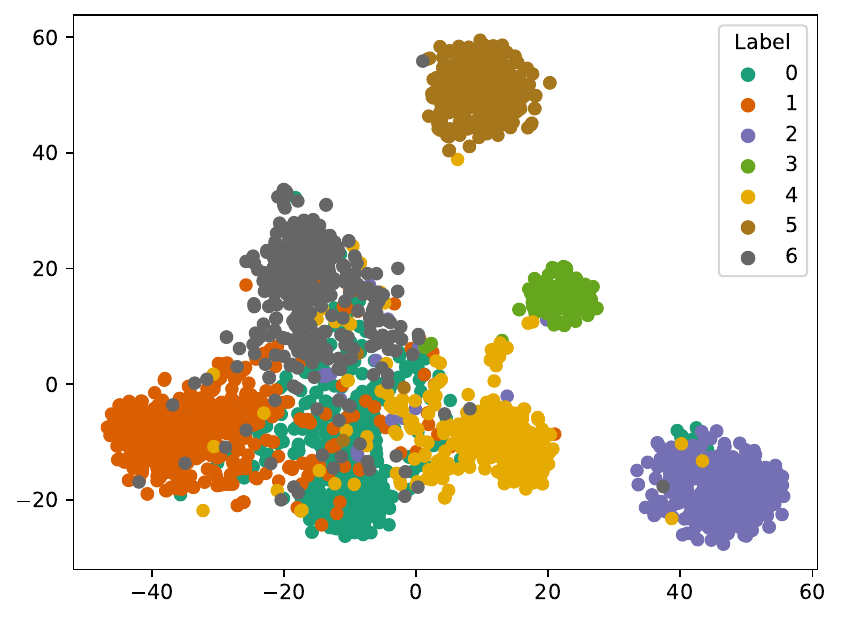}
    \end{minipage}%
    }
    \subfigure[Photo]{
    \centering
    \begin{minipage}[t]{0.23\linewidth}
	    \includegraphics[width=0.99\linewidth]{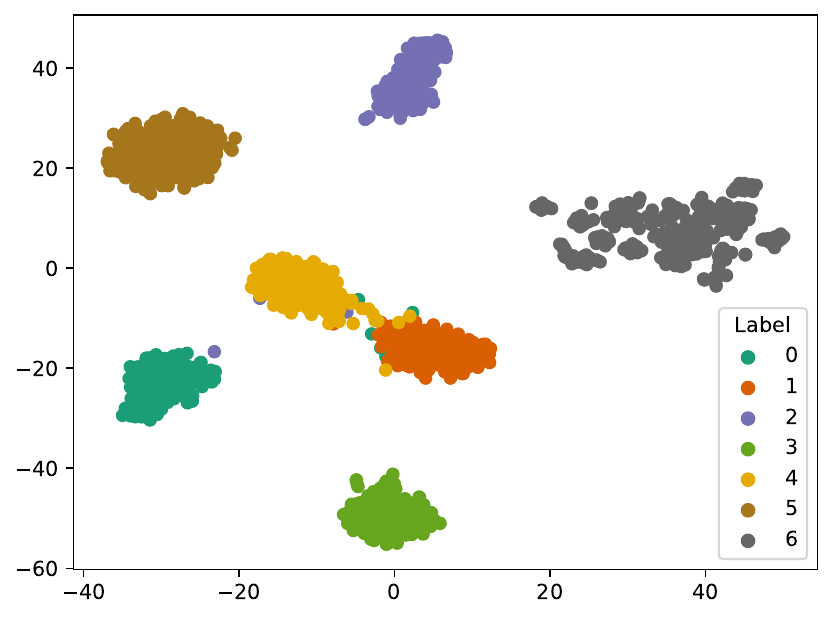}
    \end{minipage}%
    }
    \subfigure[Sketch]{
    \centering
    \begin{minipage}[t]{0.23\linewidth}
	    \includegraphics[width=0.99\linewidth]{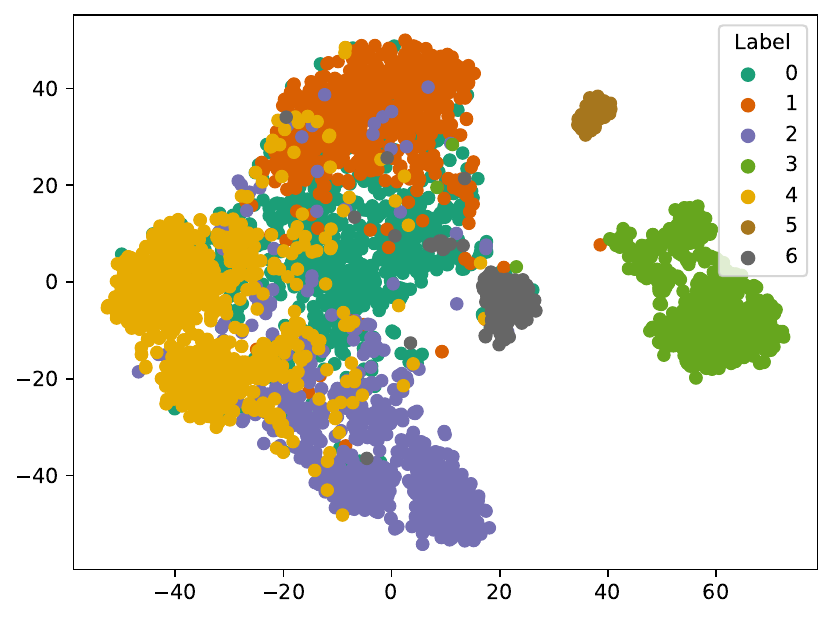}
    \end{minipage}%
    }
    \vspace{-4mm}
    \caption{t-SNE visualization of ERM embedding in four unseen test domains of PACS.}
    \label{fig:erm_embed_pacs}
    \vspace{-3mm}
\end{figure*}

\begin{figure*}[tbh!]
\setlength{\abovecaptionskip}{0mm}
\setlength{\belowcaptionskip}{-5mm}
    \centering
    \subfigure[Art Painting]{
    \centering
    \begin{minipage}[t]{0.23\linewidth}
	\includegraphics[width=1.0\linewidth]{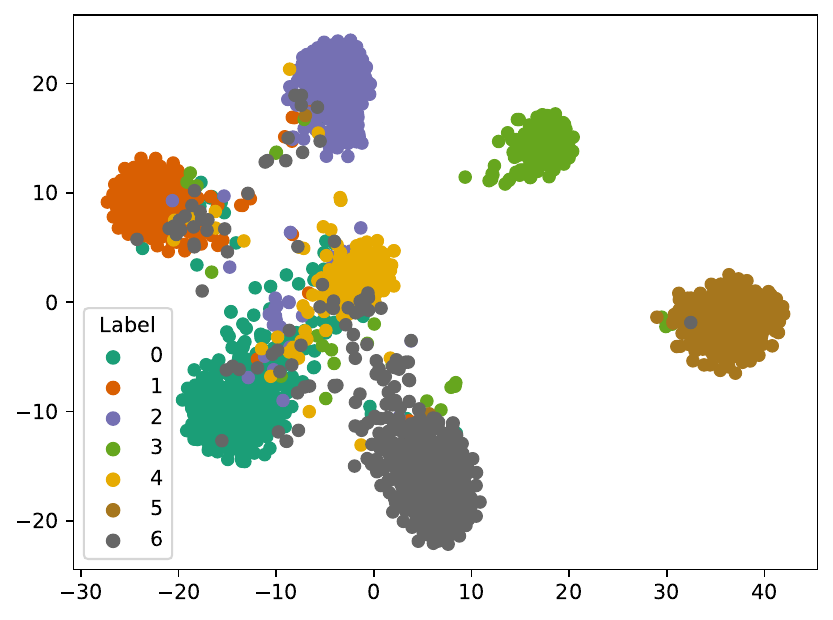}
    \end{minipage}%
    }
    \subfigure[Cartoon]{
    \centering
    \begin{minipage}[t]{0.23\linewidth}
	\includegraphics[width=1.0\linewidth]{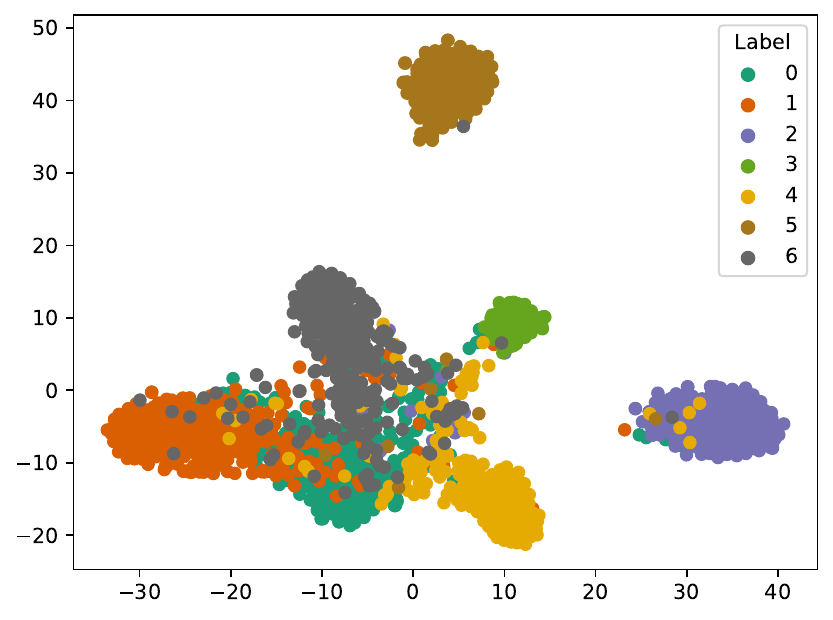}
    \end{minipage}%
    }
    \subfigure[Photo]{
    \centering
    \begin{minipage}[t]{0.23\linewidth}
	\includegraphics[width=1.0\linewidth]{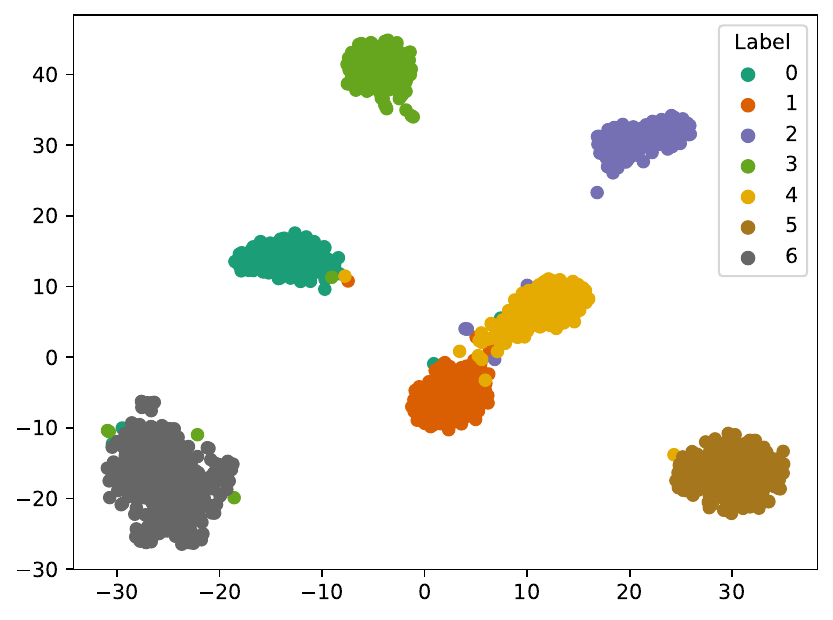}
    \end{minipage}%
    }
    \subfigure[Sketch]{
    \centering
    \begin{minipage}[t]{0.23\linewidth}
	\includegraphics[width=1.0\linewidth]{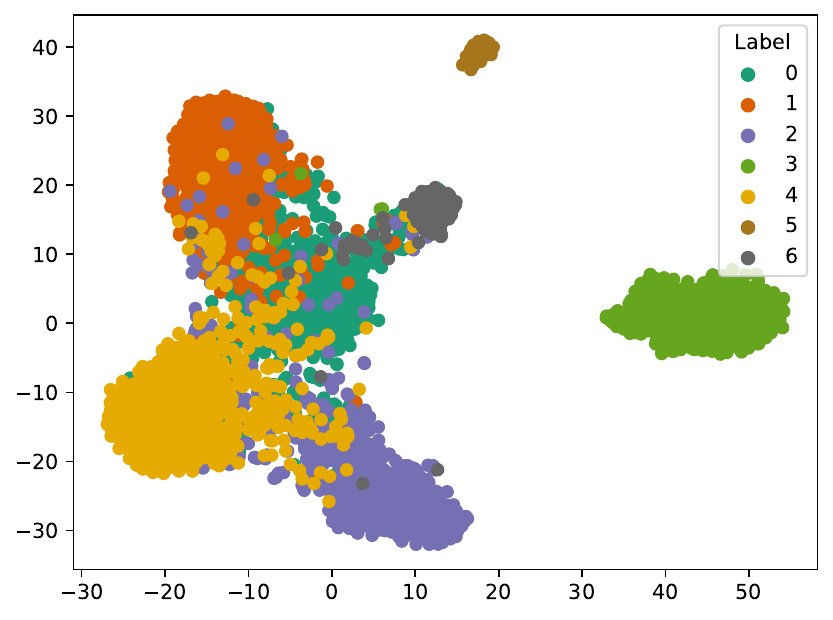}
    \end{minipage}%
    }
    \caption{t-SNE visualization of \Algnameabbr{} embedding in four unseen test domains of PACS.}
    \vspace{1mm}
    \label{fig:dispel_embed_pacs}
\end{figure*}


\noindent\textbf{Baselines.}
To fairly compare our proposed framework with existing algorithms, we follow the settings of DomainBed~\cite{gulrajani2020search} and DeepDG~\cite{wang2022generalizing}, using the best result between DomainBed, DeepDG, and the original literature.
We categorize the 12 baseline algorithms into two groups: 
$\textbf{Group 1}$: the algorithms requiring domain labels (Mixup~\cite{wang2020heterogeneous}, MLDG~\cite{li2018learning}, CORAL~\cite{sun2016deep}, MMD~\cite{li2018domain}, DANN~\cite{li2018deep}, C-DANN~\cite{li2018deep}, and DA-ERM~\cite{dubey2021adaptive}); and 
$\textbf{Group 2}$: the algorithms without requiring domain labels (ERM~\cite{vapnik1999overview}, IRM~\cite{arjovsky2019invariant}, DRO~\cite{sagawa2019distributionally}, RSC~\cite{huang2020self}, and MIRO~\cite{cha2022domain}).
More baselines details can be found in Appendix~\ref{appendix:baseline_and_implement}.

\noindent\textbf{Implementation Details.}
All the experimental results of the proposed \Algnameabbr{} are implemented and performed based on the codebase of DeepDG~\cite{wang2022generalizing}.
Regarding the setting of model selection, we use traditional \textit{training-domain validation set} for our implementation.
Since larger ResNets are known to have better generalization ability, we mainly conduct experiments with ResNet50 models for all 5 benchmark datasets, and also show the results of \Algnameabbr{} based on ResNet18 as a reference in Tab.~\ref{tab:tb2}.
For all the experimental results of \Algnameabbr{}, we employ ERM to fine-tune the pre-trained ResNet18 and ResNet50 as the fine-tuned model mentioned in Sec.~\ref{sec:dispel_frame}.
As for the EMG module in \Algnameabbr{}, we employ a pre-trained ResNet50 as the base model and set the temperature hyper-parameter $\tau$ in Eq.~\ref{eq:mask} to 0.1 for all the \Algnameabbr{} derivative models.
More implementation details of \Algnameabbr{} can be found in Appendix~\ref{appendix:baseline_and_implement}.

\subsection{Generalization Efficacy Analysis (RQ1)}


To evaluate the generalization performance of the proposed \Algnameabbr{} framework, we compare the results with 12 baseline methods on five datasets.
We summarize the results comparing \Algnameabbr{} to other baselines based on a ResNet50 pre-trained model in Tab.~\ref{tab:tb1}.
Overall, \Algnameabbr{} achieves state-of-the-art without using domain labels.
\Algnameabbr{} achieves the best average accuracy on unseen test domains in 4 out of 5 benchmark datasets.
On average, \Algnameabbr{} shows the best accuracy on unseen test domains over 5 benchmarks, meaning that it has stable effectiveness across different data distributions.
Considering the extensive experimentation conducted on 5 benchmark datasets and 22 unseen test domains, the results conclusively demonstrate the efficacy of \Algnameabbr{} in improving the diverse types of image data.

\begin{table}[tbh!]
    \small
    \renewcommand{\arraystretch}{0.7}
    \centering
    \caption{Average unseen test domain accuracy improvement for ERM (ResNet18).}
    \makebox[0.45\textwidth][c]{
        \resizebox{0.48\textwidth}{!}{
            \begin{tabular}{l c c c}
                \toprule
                 & PACS & Office-Home & Avg. \\
                \midrule
                \multicolumn{4}{c}{\textbf{Group 1}: algorithms requiring domain labels} \\
                \midrule
                Mixup~\cite{wang2020heterogeneous} & 82.3 $\pm$ 0.4 & 64.3 $\pm$ 0.2 & 73.3 \\
                CORAL~\cite{sun2016deep} & 82.8 $\pm$ 0.1 & 64.0 $\pm$ 0.3 & 73.4 \\
                MMD~\cite{li2018domain} & 83.2 $\pm$ 0.2 & 64.2 $\pm$ 0.1 & 73.7 \\
                DANN~\cite{ganin2016domain} & 83.6 $\pm$ 0.8 & 62.6 $\pm$ 0.5 & 73.1 \\
                \midrule
                \multicolumn{4}{c}{\textbf{Group 2}: algorithms without requiring domain labels} \\
                \midrule
                ERM~\cite{vapnik1999overview} & 82.1 $\pm$ 0.1 & 63.0 $\pm$ 0.1 & 72.6 \\
                DRO~\cite{sagawa2019distributionally} & 82.2 $\pm$ 0.2 & 63.9 $\pm$ 0.2 & 73.1 \\
                RSC~\cite{huang2020self} & 83.6 $\pm$ 0.2 & 63.4 $\pm$ 0.3 & 73.5 \\
                \midrule
                \Algnameabbr{} & \textbf{85.4} $\pm$ 0.1 & \textbf{67.2} $\pm$ 0.0 & \textbf{76.3} \\
                \bottomrule
            \end{tabular}
        }
    }
    \label{tab:tb2}
\end{table}

\begin{table}[tbh!]
    \renewcommand{\arraystretch}{0.9}
    \centering
    \caption{Each unseen test domain accuracy comparisons of PACS (ResNet50).}
    \makebox[0.45\textwidth][c]{
        \resizebox{0.48\textwidth}{!}{
            \begin{tabular}{l c c c c}
                \toprule
                 & Art Painting & Cartoon & Photo & Sketch \\
                \midrule
                \multicolumn{5}{c}{\textbf{Group 1}: algorithms requiring domain labels} \\
                \midrule
                Mixup~\cite{wang2020heterogeneous} & \textbf{89.3} $\pm$ 0.5 & 81.7 $\pm$ 0.1 & 97.3 $\pm$ 0.4 & 82.3 $\pm$ 0.8 \\
                MLDG~\cite{li2018learning} & 89.1 $\pm$ 0.9 & 78.8 $\pm$ 0.7 & 97.0 $\pm$ 0.9 & 74.4 $\pm$ 2.0 \\
                CORAL~\cite{sun2016deep} & 84.7 $\pm$ 0.6 & 81.5 $\pm$ 1.1 & 96.7 $\pm$ 0.0 & 81.7 $\pm$ 0.1 \\
                MMD~\cite{li2018domain} & 84.5 $\pm$ 0.6 & 79.7 $\pm$ 0.7 & 97.5 $\pm$ 0.4 & 78.1 $\pm$ 1.3 \\
                DANN~\cite{ganin2016domain} & 87.1 $\pm$ 0.5 & \textbf{83.2} $\pm$ 1.4 & 96.5 $\pm$ 0.2 & 79.8 $\pm$ 2.8 \\
                C-DANN~\cite{li2018deep} & 84.0 $\pm$ 0.9 & 78.5 $\pm$ 1.5 & 97.0 $\pm$ 0.4 & 71.8 $\pm$ 3.9 \\
                \midrule
                \multicolumn{5}{c}{\textbf{Group 2}: algorithms without requiring domain labels} \\
                \midrule
                ERM~\cite{vapnik1999overview} & 83.7 $\pm$ 0.1 & 81.8 $\pm$ 1.3 & 96.7 $\pm$ 0.0 & 83.4 $\pm$ 0.9 \\
                IRM~\cite{arjovsky2019invariant} & 85.0 $\pm$ 1.6 & 77.6 $\pm$ 0.9 & 96.7 $\pm$ 0.3 & 78.5 $\pm$ 2.6 \\
                DRO~\cite{sagawa2019distributionally} & 85.0 $\pm$ 0.3 & 81.8 $\pm$ 0.8 & 96.1 $\pm$ 0.3 & 84.3 $\pm$ 0.7 \\
                RSC~\cite{huang2020self} & 87.8 $\pm$ 0.8 & 80.3 $\pm$ 1.8 & 97.7 $\pm$ 0.3 & 81.5 $\pm$ 1.2 \\
                \midrule
                \Algnameabbr{} & 87.1 $\pm$ 0.1 & 82.5 $\pm$ 0.0 & \textbf{98.0} $\pm$ 0.1 & \textbf{85.2} $\pm$ 0.1 \\
                \bottomrule
            \end{tabular}
        }
    }
    \label{tab:tb3}
\end{table}

To evaluate the efficacy of \Algnameabbr{} for improving domain generalization performance based on a less generalization architecture, we also conduct the experiments with a ResNet18.
Although it cannot improve ResNet18-based ERM to achieve comparable performance with other ResNet50-based baselines, the results in Tab.~\ref{tab:tb2} show an impressive potential of \Algnameabbr{} for boosting the generalization ability of small pre-trained architectures.
The full experimantal results of \Algnameabbr{} can be found in Appendix~\ref{appendix:exp_results}.

\noindent\textbf{Observation 1: \Algnameabbr{} can be adopted to different neural architectures.}
Based on the results shown in Tab.~\ref{tab:tb1} and Tab.~\ref{tab:tb2}, \Algnameabbr{} maintains the generalizing efficacy with different architectures.
Even being utilized in a small pre-trained model which is known to have worse generalization ability, \Algnameabbr{} can improve the prediction accuracy in unseen domains.

\begin{figure*}[tbh!]
    \centering
    \subfigure[Location 100]{
    \centering
    \begin{minipage}[t]{0.23\linewidth}
	    \includegraphics[width=1.0\linewidth]{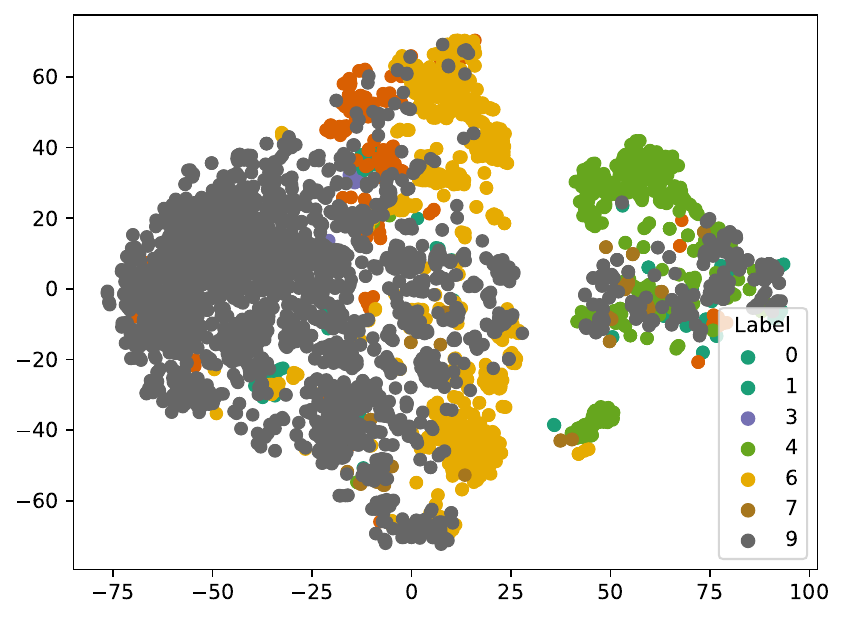}
    \end{minipage}%
    }
    \subfigure[Location 38]{
    \centering
    \begin{minipage}[t]{0.23\linewidth}
	    \includegraphics[width=0.99\linewidth]{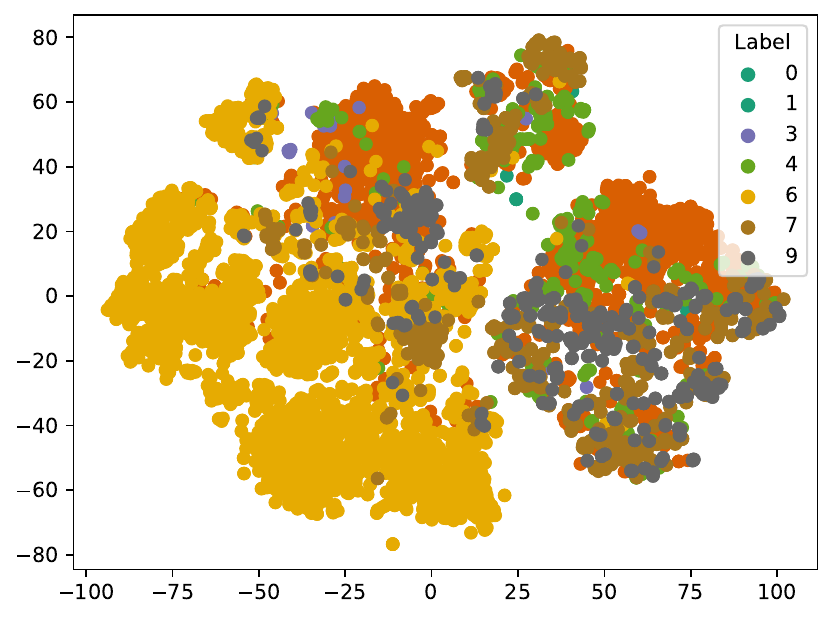}
    \end{minipage}%
    }
    \subfigure[Location 43]{
    \centering
    \begin{minipage}[t]{0.23\linewidth}
	    \includegraphics[width=0.99\linewidth]{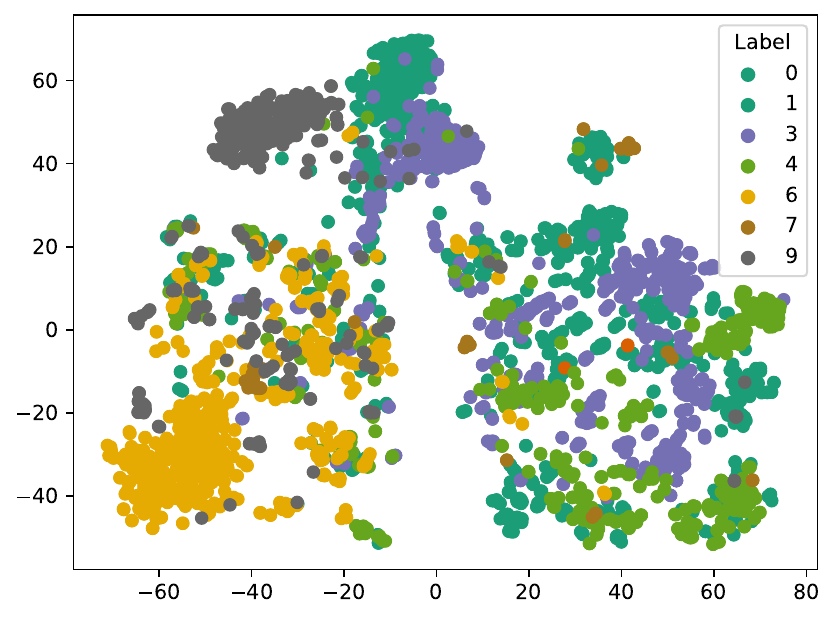}
    \end{minipage}%
    }
    \subfigure[Location 46]{
    \centering
    \begin{minipage}[t]{0.23\linewidth}
	    \includegraphics[width=0.99\linewidth]{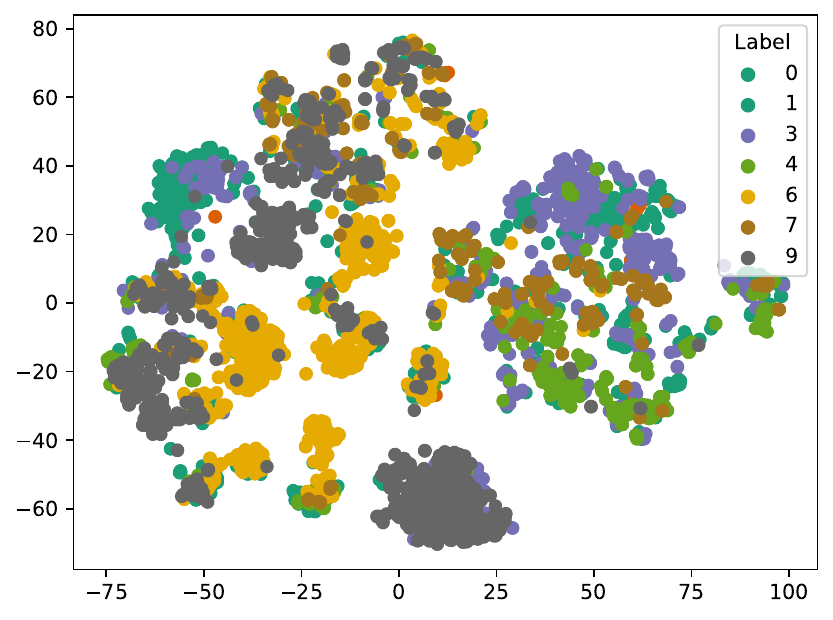}
    \end{minipage}%
    }
    \vspace{-4mm}
    \caption{t-SNE visualization of ERM embedding in four unseen test domains of Terra Incognita.}
    \label{fig:erm_embed_terra}
    \vspace{-3mm}
\end{figure*}

\begin{figure*}[tbh!]
\setlength{\abovecaptionskip}{0mm}
\setlength{\belowcaptionskip}{-5mm}
    \centering
    \subfigure[Location 100]{
    \centering
    \begin{minipage}[t]{0.23\linewidth}
	\includegraphics[width=1.0\linewidth]{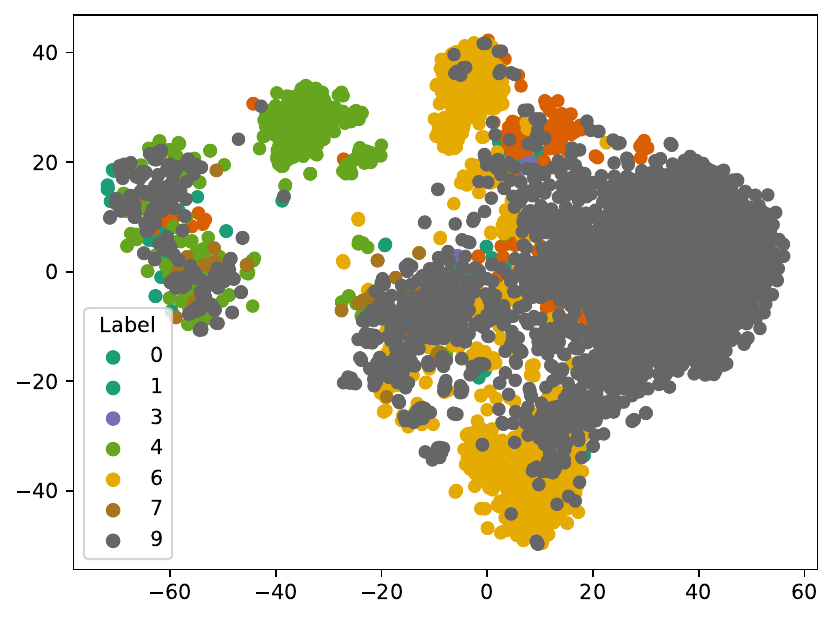}
    \end{minipage}%
    }
    \subfigure[Location 38]{
    \centering
    \begin{minipage}[t]{0.23\linewidth}
	\includegraphics[width=1.0\linewidth]{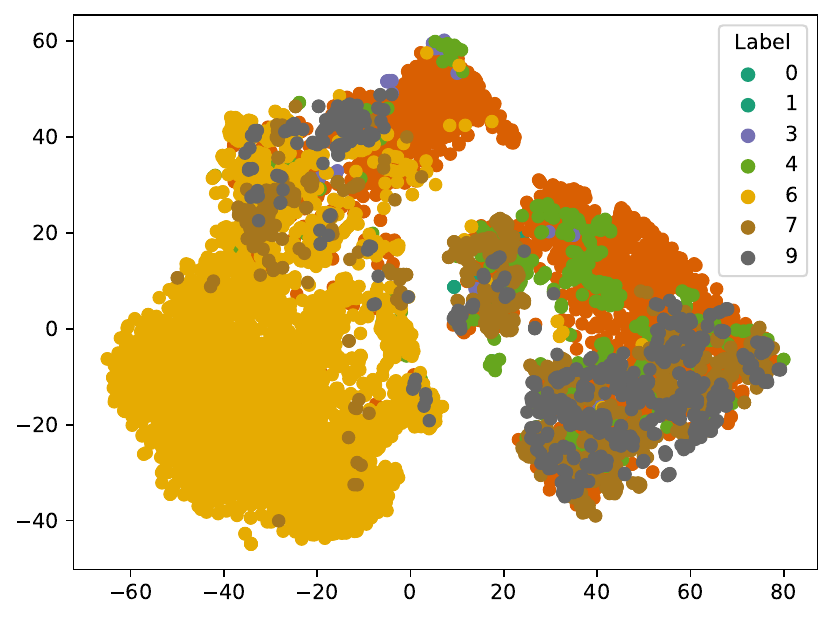}
    \end{minipage}%
    }
    \subfigure[Location 43]{
    \centering
    \begin{minipage}[t]{0.23\linewidth}
	\includegraphics[width=1.0\linewidth]{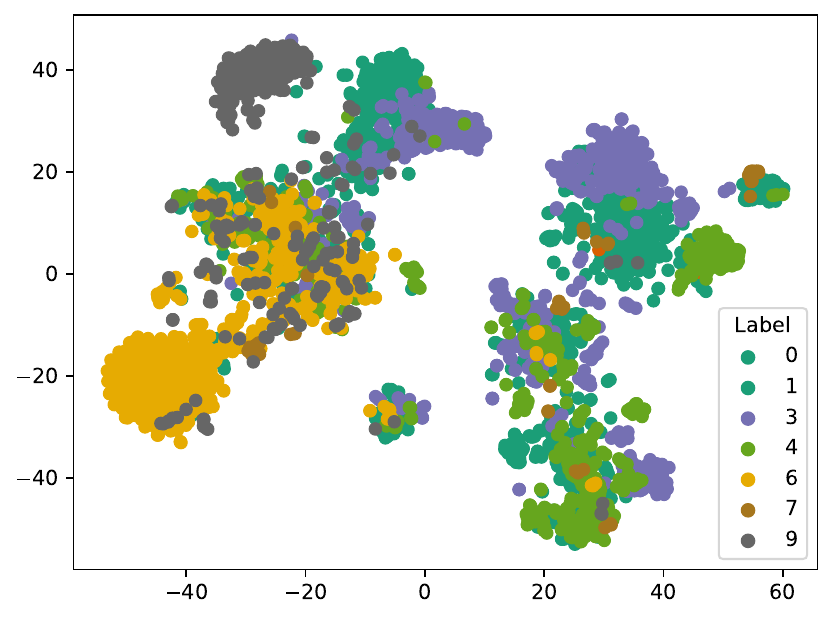}
    \end{minipage}%
    }
    \subfigure[Location 46]{
    \centering
    \begin{minipage}[t]{0.23\linewidth}
	\includegraphics[width=1.0\linewidth]{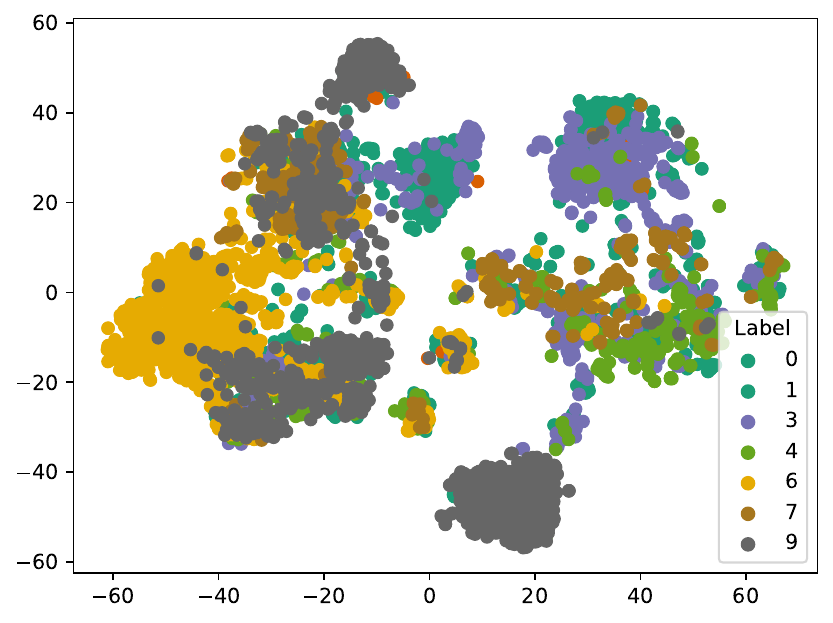}
    \end{minipage}%
    }
    \caption{t-SNE visualization of \Algnameabbr{} embedding in four unseen test domains of Terra Incognita.}
    \vspace{1mm}
    \label{fig:dispel_embed_terra}
\end{figure*}

\subsection{Manners of Domain-Specific Liberating (RQ2)}
\label{sec:rq2}


As mentioned in Sec.~\ref{sec:dispel}, \Algnameabbr{} aims to enhance the generalizability of fine-tuned models in a fine-grained instance-specific manner, which considers the unique characteristic of each instance to mask embedding.
Hence, in addition to the encouraging experimental outcomes demonstrating the generalizing efficacy of the proposed \Algnameabbr{}, we seek to delve into the detailed mechanisms underlying \Algnameabbr{}'s performance.

\begin{table}[t!]
    \Large
    \renewcommand{\arraystretch}{0.99}
    \centering
    \caption{Each unseen test domain accuracy comparisons of Terra Incognita (ResNet50).}
    \makebox[0.45\textwidth][c]{
        \resizebox{0.48\textwidth}{!}{
            \begin{tabular}{l c c c c}
                \toprule
                 & Location 100 & Location 38 & Location 43 & Location 46 \\
                \midrule
                \multicolumn{5}{c}{\textbf{Group 1}: algorithms requiring domain labels} \\
                \midrule
                Mixup~\cite{wang2020heterogeneous} & \textbf{60.6} $\pm$ 1.3 & 41.1 $\pm$ 1.8 & 58.5 $\pm$ 0.8 & 35.2 $\pm$ 1.1 \\
                MLDG~\cite{li2018learning} & 48.5 $\pm$ 3.3 & 42.8 $\pm$ 0.4 & 56.8 $\pm$ 0.9 & 36.3 $\pm$ 0.5 \\
                CORAL~\cite{sun2016deep} & 48.6 $\pm$ 0.9 & 42.2 $\pm$ 3.5 & 55.9 $\pm$ 0.6 & 38.7 $\pm$ 0.7 \\
                MMD~\cite{li2018domain} & 52.2 $\pm$ 5.8 & 47.0 $\pm$ 0.6 & \textbf{57.8} $\pm$ 1.3 & 40.3 $\pm$ 0.5 \\
                DANN~\cite{ganin2016domain} & 49.0 $\pm$ 3.8 & 46.3 $\pm$ 1.7 & 57.6 $\pm$ 0.8 & 40.6 $\pm$ 1.7 \\
                C-DANN~\cite{li2018deep} & 49.5 $\pm$ 3.8 & 44.8 $\pm$ 1.0 & 57.3 $\pm$ 1.1 & 38.8 $\pm$ 1.7 \\
                \midrule
                \multicolumn{5}{c}{\textbf{Group 2}: algorithms without requiring domain labels} \\
                \midrule
                ERM~\cite{vapnik1999overview} & 50.8 $\pm$ 0.2 & 42.5 $\pm$ 0.2 & \textbf{57.9} $\pm$ 1.3 & 37.6 $\pm$ 1.3 \\
                IRM~\cite{arjovsky2019invariant} & 44.2 $\pm$ 2.7 & 41.3 $\pm$ 0.6 & 54.3 $\pm$ 0.2 & 36.0 $\pm$ 1.7 \\
                DRO~\cite{sagawa2019distributionally} & 31.8 $\pm$ 0.3 & 43.7 $\pm$ 1.2 & 58.0 $\pm$ 0.7 & 36.6 $\pm$ 1.3 \\
                RSC~\cite{huang2020self} & 50.2 $\pm$ 2.2 & 39.2 $\pm$ 1.4 & 56.3 $\pm$ 1.4 & 40.8 $\pm$ 0.6 \\
                \midrule
                \Algnameabbr{} & 54.7 $\pm$ 0.3 & \textbf{48.1} $\pm$ 0.0 & 56.3 $\pm$ 0.3 & \textbf{42.3} $\pm$ 0.2 \\
                \bottomrule
            \end{tabular}
        }
    }
    \vspace{-0.3cm}
    \label{tab:tb4}
\end{table}

\subsubsection{OOD Performance in Each Unseen Test Domain.}
\label{sec:each_domain}

To closely investigate the manner of \Algnameabbr{}, we first observe the prediction accuracy in each unseen test domain.
\Algnameabbr{} can improve the accuracy of the base algorithm (here we use ERM) on 21 out of 22 unseen test domains, except the domain Location 43 of the Terra Incognita dataset.

Tab.~\ref{tab:tb3} shows the accuracy of each unseen test domain of PACS.
Despite only achieving the best performance in two unseen test domains, \Algnameabbr{} can increase the accuracy in all four domains compared to the base algorithm ERM. 
This phenomenon reflects the instance-specific manner of \Algnameabbr{} due to the ability of masking data representations according to each input image, which is a more fine-grained perspective than identifying different domain groups for generalization.

Regarding the Terra Incognita, the results show that \Algnameabbr{} did not enhance accuracy of ERM in the \textit{Location 43} domain, the only one out of 22 domains tested.
However, it still has the best performance on average of the four unseen test domains of Terra Incognita.
This result again mirrors that the fine-grained instance-specific embedding masking of \Algnameabbr{} brings a more stable generalizing efficacy among different out-of-distribution sets.
Considering the DoainNet as a larger-scale benchmark with a more difficult multi-classification task, \Algnameabbr{} achieves the best accuracy in five out of six unseen test domains and improve the accuracy in all the unseen domains compared to ERM, as shown in Tab.~\ref{tab:tb5}.

\noindent\textbf{Observation 2: \Algnameabbr{} possesses stable generalizing efficacy.}
The results show that \Algnameabbr{} maintains its stable efficacy in improving generalization ability over more different data distributions in more diverse classes of data. 
And these results reflect the purpose of the EMG module that considers each instance for fine-grained domain-specific feature masking.

\begin{table}[tbh!]
    \Huge
    \renewcommand{\arraystretch}{1.0}
    \centering
    \caption{Each unseen test domain accuracy comparisons of DomainNet (ResNet50).}
    \makebox[0.45\textwidth][c]{
        \resizebox{0.48\textwidth}{!}{
            \begin{tabular}{l c c c c c c}
                \toprule
                 & Clipart & Infograph & Painting & Quickdraw & Real & Sketch \\
                \midrule
                \multicolumn{7}{c}{\textbf{Group 1}: algorithms requiring domain labels} \\
                \midrule
                Mixup~\cite{wang2020heterogeneous} & 55.3 $\pm$ 0.3 & 18.2 $\pm$ 0.3 & 45.0 $\pm$ 1.0 & 12.5 $\pm$ 0.3 & 57.1 $\pm$ 1.2 & 49.2 $\pm$ 0.3 \\
                MLDG~\cite{li2018learning} & 59.5 $\pm$ 0.0 & 19.8 $\pm$ 0.4 & \textbf{48.3} $\pm$ 0.5 & 13.0 $\pm$ 0.4 & 59.5 $\pm$ 1.0 & 50.4 $\pm$ 0.7 \\
                CORAL~\cite{sun2016deep} & 58.7 $\pm$ 0.2 & \textbf{20.9} $\pm$ 0.3 & 47.3 $\pm$ 0.3 & 13.6 $\pm$ 0.3 & 60.2 $\pm$ 0.3 & 50.2 $\pm$ 0.6 \\
                MMD~\cite{li2018domain} & 54.6 $\pm$ 1.7 & 19.3 $\pm$ 0.3 & 44.9 $\pm$ 1.1 & 11.4 $\pm$ 0.5 & 59.5 $\pm$ 0.2 & 47.0 $\pm$ 1.6 \\
                DANN~\cite{ganin2016domain} & 53.8 $\pm$ 0.7 & 17.8 $\pm$ 0.3 & 43.5 $\pm$ 0.3 & 11.9 $\pm$ 0.5 & 56.4 $\pm$ 0.3 & 46.7 $\pm$ 0.5 \\
                C-DANN~\cite{li2018deep} & 53.4 $\pm$ 0.4 & 18.3 $\pm$ 0.7 & 44.8 $\pm$ 0.3 & 12.9 $\pm$ 0.2 & 57.5 $\pm$ 0.4 & 46.7 $\pm$ 0.2 \\
                \midrule
                \multicolumn{7}{c}{\textbf{Group 2}: algorithms without requiring domain labels} \\
                \midrule
                ERM~\cite{vapnik1999overview} & 58.4 $\pm$ 0.3 & 19.2 $\pm$ 0.4 & 46.3 $\pm$ 0.5 & 12.8 $\pm$ 0.0 & 60.6 $\pm$ 0.5 & 49.7 $\pm$ 0.8 \\
                IRM~\cite{arjovsky2019invariant} & 51.0 $\pm$ 3.3 & 16.8 $\pm$ 1.0 & 38.8 $\pm$ 2.1 & 11.8 $\pm$ 0.5 & 51.5 $\pm$ 3.6 & 44.2 $\pm$ 3.1 \\
                DRO~\cite{sagawa2019distributionally} & 47.8 $\pm$ 0.6 & 17.1 $\pm$ 0.6 & 36.6 $\pm$ 0.7 & 8.8 $\pm$ 0.4 & 51.5 $\pm$ 0.6 & 40.7 $\pm$ 0.3 \\
                RSC~\cite{huang2020self} & 55.0 $\pm$ 1.2 & 18.3 $\pm$ 0.5 & 44.4 $\pm$ 0.6 & 12.2 $\pm$ 0.6 & 55.7 $\pm$ 0.7 & 47.8 $\pm$ 0.9 \\
                \midrule
                \Algnameabbr{} & \textbf{63.4} $\pm$ 0.1 & 20.1 $\pm$ 0.0 & \textbf{48.2} $\pm$ 0.0 & \textbf{14.2} $\pm$ 0.0 & \textbf{63.4} $\pm$ 0.0 & \textbf{54.9} $\pm$ 0.0 \\
                \bottomrule
            \end{tabular}
        }
    }
    \label{tab:tb5}
\end{table}

\subsubsection{Visualization Analysis via t-SNE}

To illustrate how \Algnameabbr{} improves generalization by blocking domain-specific features in the embedding space, we use t-SNE in Fig.~\ref{fig:erm_embed_pacs} to ~\ref{fig:dispel_embed_terra} in the unseen test domains of PACS and Terra Incognita by comparing the embedding with and without \Algnameabbr{}.

Comparing Fig.~\ref{fig:erm_embed_pacs}-(a) and Fig.~\ref{fig:dispel_embed_pacs}-(a), the key observation is that \Algnameabbr{} aims to make each class more concentrated and separate them better.
By drawing down more precise decision boundaries, the predictor can achieve better accuracy in the unseen \textit{Art Painting} domain, in which \Algnameabbr{} enhances the most accuracy among the 4 domains as shown in Tab.~\ref{tab:tb3}.
Even when the boundaries between classes are blurred and hard to be separated, \Algnameabbr{} keeps its impact by making them more concentrated inside each class.
For instance, the results on \textit{Sketch} domain, as shown in Fig.~\ref{fig:erm_embed_pacs}-(d) and Fig.~\ref{fig:dispel_embed_pacs}-(d), show that \Algnameabbr{} concentrates the representation of each class, decreasing the length of the boundaries between class 0 and 1, 2, and 4.
By decreasing the length of boundaries between different groups, the vague part for classification will be less than the original distribution, which means that the correctness of the predictions in the unseen test domains will increase.

Investigating the embedding of Terra Incognita, a more difficult multi-class classification task dataset, we observe the coherent behavior of \Algnameabbr{} to its manner in PACS.
As shown in Fig.~\ref{fig:erm_embed_terra}-(a)(b)(d) and Fig.~\ref{fig:dispel_embed_terra}-(a)(b)(d), \Algnameabbr{} has the same effect as on \textit{Cartoon} and \textit{Sketch} domain of PACS, which is to reduce the length of decision boundaries between different classes by concentrating distribution of each class.
In addition, as shown in Tab.~\ref{tab:tb4}, \textit{Location 43} is the only domain in which \Algnameabbr{} cannot improve its classification accuracy.
However, the reason is that we cannot achieve our reproduced ERM the same performance as provided in DomainBed~\cite{gulrajani2020search}, and the accuracy of our reproduced ERM in the \textit{Location 43} domain is 55.1\%.
Therefore, \Algnameabbr{} actually improves the accuracy in this unseen test domain by 1.3\%.
As we can see in Fig.~\ref{fig:erm_embed_terra}-(c) and Fig.~\ref{fig:dispel_embed_terra}-(c), each class's instance embedding is concentrated after employing \Algnameabbr{} as in other domains.

\noindent\textbf{Observation 3: \Algnameabbr{} concentrates the distribution of each class embedding.}
The t-SNE analysis demonstrates the superiority of \Algnameabbr{}, which improves the domain generalization ability of the fine-tuned ERM by concentrating the distribution of embeddings in the same class.
Due to this manner, \Algnameabbr{} can improve classification accuracy in unseen domains, as shown in Tab.~\ref{tab:tb1} to Tab.~\ref{tab:tb5} and Appendix~\ref{appendix:exp_results}.

\begin{table}[t!]
    \small
    \renewcommand{\arraystretch}{0.6}
    \centering
    \caption{Average unseen test domain accuracy of \Algnameabbr{} by using other domain generalization algorithms as baselines. It indicates that \Algnameabbr{} can boost generalization of other algorithms.}
    \vspace{-0.2cm}
    \makebox[0.45\textwidth][c]{
        \resizebox{0.48\textwidth}{!}{
            \begin{tabular}{l c c c}
                \toprule
                 & PACS & Office-Home & Avg. \\
                \midrule
                ERM~\cite{vapnik1999overview} & 86.4 $\pm$ 0.1 & 69.9 $\pm$ 0.1 & 78.2 \\
                ERM w/ \Algnameabbr{} & \textbf{88.2} $\pm$ 0.1 & \textbf{73.3} $\pm$ 0.3 & \textbf{80.8} \\
                \midrule
                DRO~\cite{sagawa2019distributionally} & 86.8 $\pm$ 0.4 & 70.2 $\pm$ 0.3 & 78.5 \\
                DRO w/ \Algnameabbr{} & \textbf{88.2} $\pm$ 0.1 & \textbf{72.2} $\pm$ 0.1 & \textbf{80.2} \\
                \midrule
                CORAL~\cite{sun2016deep} & 86.2 $\pm$ 0.2 & 70.1 $\pm$ 0.4 & 78.2 \\
                CORAL w/ \Algnameabbr{} & \textbf{87.5} $\pm$ 0.1 & \textbf{72.8} $\pm$ 0.2 & \textbf{80.2} \\
                \midrule
                DANN~\cite{ganin2016domain} & 86.7 $\pm$ 1.1 & 69.5 $\pm$ 0.6 & 78.1 \\
                DANN w/ \Algnameabbr{} & \textbf{87.2} $\pm$ 0.0 & \textbf{72.9} $\pm$ 0.3 & \textbf{80.1} \\
                \midrule
                Mixup~\cite{wang2020heterogeneous} & 87.7 $\pm$ 0.5 & 71.2 $\pm$ 0.1 & 79.5 \\
                Mixup w/ \Algnameabbr{} & \textbf{89.1} $\pm$ 0.1 & \textbf{73.4} $\pm$ 0.1 & \textbf{81.3} \\
                \bottomrule
            \end{tabular}
        }
    }
    \vspace{-0.3cm}
    \label{tab:alg_w_dispel}
\end{table}

\subsection{Boosting Other Algorithm via \Algnameabbr{} (RQ3)}
\label{sec:boosting}

In previous experiments, we adopt fine-tuned ERM to be the base model of EMG module in \Algnameabbr{}.
Based on the experimental results showing that the proposed \Algnameabbr{} framework can improve the prediction performance in unseen domains, we are curious if \Algnameabbr{} can also enhance the generalization ability of other algorithms.

To investigate the effectiveness of the proposed \Algnameabbr{} for other approaches, we conduct experiments by comparing it with different baselines.
Specifically, we employ \Algnameabbr{} for 3 representation learning methods; 1 does not require domain labels (DRO), and the others require domain labels (CORAL and DANN).
We also adopt \Algnameabbr{} for 1 data manipulation method that requires domain labels for training (Mixup).
According to Tab.~\ref{tab:alg_w_dispel}, \Algnameabbr{} can improve the prediction accuracy of all 4 baselines on unseen test domains in the 2 benchmarks.
However, the improvements are not consistent across different algorithms.
In PACS, the improvement of DANN is insignificant compared with other methods, and the boosts of all four algorithms are less than the increase of ERM.
Nevertheless, the improvement of DANN in Office-Home is more significant than all other algorithms.
It makes sense when we consider the encoders of each fine-tuned algorithm as a different initial state for training the base model of the EMG component. 
The reason is that the initial states influence the optimization process and might lead to different model weights.

Furthermore, we observe that the higher initial accuracy of an algorithm does not guarantee better performance after leveraging \Algnameabbr{} for boosting.
In PACS, the fine-tuned Mixup model has the best domain generalization performance compared with the other four fine-tuned models, and \Algnameabbr{} then boosts it to achieve the highest accuracy in unseen test domains.
However, despite the different initial performances of five algorithms in PACS, the accuracy after \Algnameabbr{} boosting of ERM and DRO achieves the same level when CORAL and DANN cannot reach them.
Besides, by leveraging \Algnameabbr{} in the Office-Home dataset, DANN achieves performance close to ERM and Mixup when DRO performance is slightly lower than other methods.

\noindent\textbf{Observation 4: \Algnameabbr{} can improve generalization ability for different types of algorithms.}
The experimental results of utilizing \Algnameabbr{} in four algorithms demonstrate that \Algnameabbr{} can improve the generalization performance of fine-tuned models.
According to the results shown in Tab.~\ref{tab:alg_w_dispel}, despite Mixup achieving the best performance with \Algnameabbr{}, ERM can achieve equivalent results via \Algnameabbr{} without using domain labels during training.

\section{Conclusions and Future Work}

In this work, we demonstrate the efficacy of masking domain-specific features in embedding space for improving the generalization ability of a fine-tuned prediction model.
Based on this observation, we propose a post-processing fine-grained masking framework, \Algnameabbr{}, by accounting for instance discrepancies to further improve generalization ability.
Specifically, \Algnameabbr{} uses a mask generator that generates a distinct mask for each input data, which is used to filter domain-specific features in embedding space.
The results on five benchmarks demonstrate that \Algnameabbr{} outperforms the state-of-the-art baselines.
Regarding future directions, we plan to explore the potential of exploiting \Algnameabbr{} for different downstream tasks and various types of input data. To achieve this, we will consider the characteristics of each task and input data, allowing us to design an objective loss that is tailored to their specific requirements.


\newpage

{\small
\bibliographystyle{ieee_fullname}
\bibliography{egbib}
}

\clearpage
\appendix

\section*{Appendix}

\section{Proof of Theorem 1}
\label{appendix:proof_of_theorem1}
\setcounter{theorem}{0}
\begin{theorem}[\textbf{Generalization Error Bound}] Let $\widetilde{\boldsymbol{x}}_k^{T}$ be a masked instance of $\boldsymbol{x}_k^T$ on an unseen domain $T$. Given an instance embedding $\boldsymbol{z}_k^T$ satisfies the composition of domain-specific $\boldsymbol{z}_k^{\text{T-sh}}$ and domain-sharing $\boldsymbol{z}_k^{\text{T-sp}}$, where $\hat{f}(\boldsymbol{x}_k^T) \!=\! \hat{c}(\boldsymbol{z}_k^T)$ be the predicted outcomes. For arbitrary distance function $d(\cdot,\cdot)$ in metric space, the generalization error $\textbf{GE} = \mathbf{E}_{\mathcal{X}}[ ~\textit{d}\big(f(\boldsymbol{x}_k^T), \hat{f}(\widetilde{\boldsymbol{x}}_k^{T}) \big)~ ]$ of \Algnameabbr{} framework can be bounded as:
\begin{equation} 
    \small
    \textbf{GE}
    \leq \mathbf{E}_{\mathcal{X}}[\textit{d}\big(c(\boldsymbol{z}_k^{\text{T-sh}}), \hat{c}(\widetilde{\boldsymbol{z}}_{k}^{\text{T-sh}}) \big)]
    + \mathbf{E}_{\mathcal{X}}[\textit{d}\big(c(\boldsymbol{z}_k^{\text{T-sp}}), \hat{c}(\widetilde{\boldsymbol{z}}_{k}^{\text{T-sp}}) \big)] 
\label{eq:thm_1}
\end{equation}
where $\widetilde{\boldsymbol{z}}_{k}^{\text{T}} \!=\! \boldsymbol{z}_k^{T} \odot m_k^T$ is composed of remained domain-specific embedding $\widetilde{\boldsymbol{z}}_{k}^{\text{T-sp}}$ and preserved domain-sharing embedding $\widetilde{\boldsymbol{z}}_{k}^{\text{T-sh}}$.
\label{apx:thm:1} 
\end{theorem} 

\begin{proof}
    In order to estimate the generalization error of \Algnameabbr{} on unseen domain $\mathcal{T}$, we calculate the expected values of distance between $f(\boldsymbol{x}_k^T)$ and $ \hat{f}(\widetilde{\boldsymbol{x}}_k^{T})$. Hence, the estimated generalization error of $\mathcal{T}$ can be elaborated as:
    \begin{align}
        \notag \textbf{GE} &= \mathbf{E}_{\mathcal{T}}[ ~\textit{d}\big(f(\boldsymbol{x}_k^T), \hat{f}(\widetilde{\boldsymbol{x}}_k^{T}) \big)~ ] \\
        &= \int_{\mathcal{X}} \textit{d}\big(f(\boldsymbol{x}_k^T), \hat{f}(\widetilde{\boldsymbol{x}}_k^{T}) \big) P(\mathcal{T}) \,d\mathcal{T}
        \label{apx:eq:ge}
    \end{align}
    where $P(\mathcal{T})$ denotes cumulative distribution function of $\mathcal{T}$.
    As a lower generalized error $\textbf{GE}$ represents better generalization capability, we can observe from Eq.~\ref{apx:eq:ge} that the closer $\textit{d}\big(f(\boldsymbol{x}_k^T), \hat{f}(\widetilde{\boldsymbol{x}}_k^{T}) \big)$ approaches zero, the better generalization capability is obtained. 
    
    In this manner, we now discus the upper bound of $\textit{d}\big(f(\boldsymbol{x}_k^T), \hat{f}(\widetilde{\boldsymbol{x}}_k^{T}) \big)$. This also ensure the upper bound of $\textbf{GE}$. Following the properties that each instance's embedding $\boldsymbol{z}_k^T$ can be composed of domain-specific $\boldsymbol{z}_k^{\text{T-sh}}$ and domain-sharing $\boldsymbol{z}_k^{\text{T-sp}}$, we consider upper bound as follows,
    \begin{align}
        \notag &\textit{d}\big(f(\boldsymbol{x}_k^T), \hat{f}(\widetilde{\boldsymbol{x}}_k^{T}) \big) \\
        \notag &= \textit{d}\big(c(\boldsymbol{z}_k^T)~,~ \hat{f}(\widetilde{\boldsymbol{z}}_k^{T}) \big) \\
        &= \textit{d}\big(c(\boldsymbol{z}_k^{\text{T-sh}} + \boldsymbol{z}_k^{\text{T-sp}})~,~ \hat{c}( \widetilde{\boldsymbol{z}}_k^{\text{T-sh}} + \widetilde{\boldsymbol{z}}_k^{\text{T-sp}} \big)
        \label{apx:eq:dis}
    \end{align}
    Since $c(\cdot)$ is the linear predictor, we can now recast the Eq.~\ref{apx:eq:dis} in the following,
    \begin{align}
        \notag &\textit{d}\big(f(\boldsymbol{x}_k^T), \hat{f}(\widetilde{\boldsymbol{x}}_k^{T}) \big) \\
        \notag &= \textit{d}\big(c(\boldsymbol{z}_k^{\text{T-sh}} + \boldsymbol{z}_k^{\text{T-sp}})~,~ \hat{c}( \widetilde{\boldsymbol{z}}_k^{\text{T-sh}} + \widetilde{\boldsymbol{z}}_k^{\text{T-sp}}\big) \\
        \notag &=  \textit{d}\big(c(\boldsymbol{z}_k^{\text{T-sh}}) + c(\boldsymbol{z}_k^{\text{T-sp}})~,~ \hat{c}( \widetilde{\boldsymbol{z}}_k^{\text{T-sh}} ) + \hat{c}(\widetilde{\boldsymbol{z}}_k^{\text{T-sp}}) \big) \\
        & \leq \textit{d}\big(c(\boldsymbol{z}_k^{\text{T-sh}}),  \hat{c}( \widetilde{\boldsymbol{z}}_k^{\text{T-sh}}) \big) + \textit{d}\big(c(\boldsymbol{z}_k^{\text{T-sp}}),  \hat{c}( \widetilde{\boldsymbol{z}}_k^{\text{T-sp}}) \big)
        \label{apx:eq:bound}
    \end{align}

    Following the conclusion of Eq.~\ref{apx:eq:ge} and Eq.~\ref{apx:eq:bound},  we have the upper bound of $\textbf{GE}$ as follows:
    \begin{align}
        \small
        \notag &\textbf{GE} = \mathbf{E}_{\mathcal{T}}[ ~\textit{d}\big(f(\boldsymbol{x}_k^T), \hat{f}(\widetilde{\boldsymbol{x}}_k^{T}) \big)~ ] \\
        \notag &= \int_{\mathcal{X}} \textit{d}\big(f(\boldsymbol{x}_k^T), \hat{f}(\widetilde{\boldsymbol{x}}_k^{T}) \big) P(\mathcal{T}) \,d\mathcal{T} \\
        \notag &\leq \int_{\mathcal{X}} \!\left[ \textit{d}\big(c(\boldsymbol{z}_k^{\text{T-sh}}),  \hat{c}( \widetilde{\boldsymbol{z}}_k^{\text{T-sh}}) \big) + \textit{d}\big(c(\boldsymbol{z}_k^{\text{T-sp}}),  \hat{c}( \widetilde{\boldsymbol{z}}_k^{\text{T-sp}}) \big) \right] \!P(\mathcal{T})d\mathcal{T} \\
        \notag & = \int_{\mathcal{X}} \textit{d}\big(c(\boldsymbol{z}_k^{\text{T-sh}}),  \hat{c}( \widetilde{\boldsymbol{z}}_k^{\text{T-sh}}) \big) \ P(\mathcal{T}) \ d\mathcal{T} \ +\\
        \notag & \ \ \ \ \ \ \ \ \  \ \ \ \ \int_{\mathcal{X}} \textit{d}\big(c(\boldsymbol{z}_k^{\text{T-sp}}),  \hat{c}( \widetilde{\boldsymbol{z}}_k^{\text{T-sp}})\big) \ P(\mathcal{T}) \ d\mathcal{T} \\
        \notag & = \mathbf{E}_{\mathcal{X}}[\textit{d}\big(c(\boldsymbol{z}_k^{\text{T-sh}}), \hat{c}(\widetilde{\boldsymbol{z}}_{k}^{\text{T-sh}}) \big)] + \mathbf{E}_{\mathcal{X}}[\textit{d}\big(c(\boldsymbol{z}_k^{\text{T-sp}}), \hat{c}(\widetilde{\boldsymbol{z}}_{k}^{\text{T-sp}}) \big)] 
    \end{align}

\end{proof}

\section{Related Works}
\label{appendix:related_works}

There are two primary branches of research in the field of domain generalization: data manipulation and representation learning.

\noindent \textbf{Data Manipulation.} The data manipulation branch aims to reduce overfitting by increasing the diversity and quantity of available training data. This is typically achieved through the use of data augmentation methods or generative models~\cite{tobin2017domain, peng2018sim, tremblay2018training, volpi2018generalizing, zhang2017mixup, xu2020adversarial, yan2020improve, wang2020heterogeneous}.

\noindent \textbf{Representation Learning.} Representation learning is another branch of methods that focuses on training an encoder that maps samples to a latent space where the embedding remains invariant to various domains~\cite{arjovsky2019invariant, sagawa2019distributionally, huang2020self, li2018learning, li2018domain, ganin2016domain, cha2022domain}.
Alternative approaches for achieving invariant learning have been proposed, including techniques such as correlation alignment~\cite{sun2016deep}, class-conditional adversarial learning~\cite{li2018domaincidg}, minimizing maximum mean discrepancy~\cite{li2018deep}, and mutual information regularization~\cite{cha2022domain} that doesn't require domain labels.

\noindent \textbf{Ensemble Learning.} There are some ensemble approaches for domain generalization, which train multiple models and then combine the predictions of these models at validation time to obtain a most generalization model. For instance, SWAD~\cite{cha2021swad} aims to find a flatter minima and suffers less from overfitting than vanilla SWA~\cite{izmailov2018averaging} by a dense and overfit-aware stochastic weight sampling strategy; EoA~\cite{arpit2022ensemble} finds that an
ensemble of moving average models outperforms a traditional ensemble of unaveraged models.


\section{Datasets Details}
\label{appendix:datasets}

To compare the efficacy of our proposed framework with existing algorithms, we conduct our experiments on 5 real-world benchmark datasets: PACS~\cite{li2017deeper}, Office-Home~\cite{venkateswara2017deep}, VLCS~\cite{fang2013unbiased}, Terra Incognita~\cite{beery2018recognition}, and DomainNet~\cite{peng2019moment}.
Specifically, PACS includes four image styles (Photo, Art, Cartoon, and Sketch), which are considered 4 different domains, and each domain has 7 classes of images (Dog, Elephant, Giraffe, Horse, Person, Guitar, and House) for training and testing.
It contains a total of $9,991$ instances in 4 domains.
Office-Home consists of 65 classes of images for training and testing.
These images belong to four image styles (Art, Clipart, Product, Real) being considered as 4 different domains.
It contains a total of $15,588$ instances in 4 domains.
VLCS includes images collected from 4 different datasets (Caltech101, LabelMe, SUN09, and VOC2007), which are considered 4 different domains, and each domain has 5 classes (Dog, Bird, Person, Car, and Chair) for training and testing.
It contains a total of $10,729$ instances in 4 domains.
Terra Incognita consists of 10 classes of photographs of wild animals taken at 4 different locations (Location 100, Location 38, Location 43, and Location 46), considered as 4 different domains.
For our experiments, we use the downloader of DomainBed~\cite{gulrajani2020search} to download the same version Terra Incognita dataset as theirs.
It contains a total of $24,788$ instances in 4 domains.
DomainNet includes 6 image styles (Clipart, Infograph, Painting, Quickdraw, Real, Sketch) considered as 6 different domains.
In each domain, there are 345 classes for training and testing.
It contains a total of $586,575$ instances in 6 domains.
DomainNet can be considered a larger-scale dataset with a more difficult multi-classification task than the other 4 benchmarks.

\section{Baselines and Implementation Details}
\label{appendix:baseline_and_implement}

\textbf{Baselines.}
To fairly compare our proposed framework with existing algorithms, we follow the settings of DomainBed~\cite{gulrajani2020search} and DeepDG~\cite{wang2022generalizing}, using the best result between DomainBed, DeepDG, and the original literature.
The comparisons include 12 baseline algorithms: ERM~\cite{vapnik1999overview}, IRM~\cite{arjovsky2019invariant}, DRO~\cite{sagawa2019distributionally}, RSC~\cite{huang2020self}, Mixup~\cite{wang2020heterogeneous}, MLDG~\cite{li2018learning}, CORAL~\cite{sun2016deep}, MMD~\cite{li2018domain}, DANN~\cite{ganin2016domain}, C-DANN~\cite{li2018deep}, DA-ERM~\cite{dubey2021adaptive}, and MIRO~\cite{cha2022domain}.
Considering that domain labels can be leveraged as additional information for learning representations mitigating domain-specific features projected to embedding space, we categorize the 12 baseline algorithms into two groups: 
$\textbf{Group 1}$: the algorithms requiring domain labels (Mixup, MLDG, CORAL, MMD, DANN, C-DANN, and DA-ERM); and 
$\textbf{Group 2}$: the algorithms without requiring domain labels (ERM, IRM, DRO, RSC, and MIRO).

Note that in Tab.~\ref{tab:tb2}, except the results of ERM and our \Algnameabbr{} are reproduced based on DeepDG~\cite{wang2022generalizing}, the results of other baselines are provided by the GitHub of the same survey paper.

\textbf{Implementation.}
All the experimental results of the proposed \Algnameabbr{} are implemented and performed based on the codebase of DeepDG~\cite{wang2022generalizing}.
Unlike DomainBed~\cite{gulrajani2020search}, our implementation does not use any data augmentation during training.
Regarding the setting of model selection, we use traditional \textit{training-domain validation set} for our implementation, which does not require utilizing domain labels to split the desired validation set.
For all the experimental results of \Algnameabbr{}, we employ ERM algorithm to fine-tune the ResNet-18 and ResNet-50 as the fine-tuned model mentioned in Sec.~\ref{sec:dispel_frame}.
Concerning the use of the EMG, we utilize ResNet50 as the base model for EMG since the 5 domain generalization benchmarks we tested are image datasets.

\textbf{DNN Architectures.}
The experimental results are all fine-tuned on the basis of ResNets.
Since larger ResNets are known to have better generalization ability, we mainly conduct experiments with ResNet-50 models for all 5 benchmark datasets, and we also conduct the results of \Algnameabbr{} based on ResNet-18 as a reference shown in Tab.~\ref{tab:tb2}.
For both the two base network architectures, we both use the ResNet-18 and ResNet-50 pre-trained on ImageNet.
As for the EMG component in \Algnameabbr{}, we employ a ResNet50 pre-trained on ImageNet as the base model.

\begin{table}[t!]
    \renewcommand{\arraystretch}{1.}
    \centering
    \caption{Hyper-parameters of \Algnameabbr{} based on ERM.}
    \makebox[0.45\textwidth][c]{
        \resizebox{0.48\textwidth}{!}{
            \begin{tabular}{l c c c c c}
                \toprule
                 & PACS & Office-Home & VLCS & TerraInc & DomainNet \\
                \midrule
                \multicolumn{6}{c}{\textbf{DNN Architecture}: ResNet-18} \\
                \midrule
                Batch size & 128 & 128 & 128 & 128 & 128 \\
                Learning rate & $1 \times 10^{-3}$ & $1 \times 10^{-3}$ & $3 \times 10^{-4}$ & $1 \times 10^{-4}$ & $1 \times 10^{-4}$ \\
                $\tau$ & 0.1 & 0.1 & 0.1 & 0.1 & 0.1 \\
                \midrule
                \multicolumn{6}{c}{\textbf{DNN Architecture}: ResNet-50} \\
                \midrule
                Batch size & 64 & 64 & 64 & 64 & 64 \\
                Learning rate & $5 \times 10^{-5}$ & $1 \times 10^{-3}$ & $1 \times 10^{-3}$ & $2 \times 10^{-4}$ & $1 \times 10^{-4}$ \\
                $\tau$ & 0.1 & 0.1 & 0.1 & 0.1 & 0.1 \\
                \bottomrule
            \end{tabular}
        }
    }
    \label{tab:hyper-para}
\end{table}

\begin{table}[h!]
    \small
    \renewcommand{\arraystretch}{.7}
    \centering
    \caption{Hyper-parameters of \Algnameabbr{} for boosting other algorithms, where the DNN architecture is ResNet-50.}
    \makebox[0.45\textwidth][c]{
        \resizebox{0.48\textwidth}{!}{
            \begin{tabular}{l c c c c}
                \toprule
                 & DRO & CORAL & DANN & Mixup \\
                \midrule
                \multicolumn{5}{c}{\textbf{Dataset}: PACS} \\
                \midrule
                Batch size & 64 & 64 & 64 & 64 \\
                Learning rate & $1 \times 10^{-3}$ & $1 \times 10^{-3}$ & $5 \times 10^{-3}$ & $5 \times 10^{-4}$ \\
                $\tau$ & 0.1 & 0.1 & 0.1 & 0.1 \\
                \midrule
                \multicolumn{5}{c}{\textbf{Dataset}: Office-Home} \\
                \midrule
                Batch size & 64 & 64 & 64 & 64 \\
                Learning rate & $1 \times 10^{-3}$ & $1 \times 10^{-3}$ & $1 \times 10^{-3}$ & $1 \times 10^{-3}$ \\
                $\tau$ & 0.1 & 0.1 & 0.1 & 0.1 \\
                \bottomrule
            \end{tabular}
        }
    }
    \label{tab:hyper-para-derivative}
\end{table}

\begin{table*}[t!]
    \small
    \renewcommand{\arraystretch}{0.8}
    \centering
    \caption{Each unseen test domain accuracy of \Algnameabbr{}.}
    \makebox[0.45\textwidth][c]{
        \resizebox{0.9\textwidth}{!}{
            \begin{tabular}{l c c c c c c}
                \toprule
                \multicolumn{7}{c}{\textbf{Dataset}: PACS} \\
                \midrule
                 & Art Painting & Cartoon & Photo & Sketch & - & - \\
                \midrule
                \Algnameabbr{} (ResNet-18) & 83.6 $\pm$ 0.3 & 79.0 $\pm$ 0.2 & 97.0 $\pm$ 0.0 & 81.8 $\pm$ 0.0 & - & - \\
                \Algnameabbr{} (ResNet-50) & 87.1 $\pm$ 0.1 & 82.5 $\pm$ 0.0 & 98.0 $\pm$ 0.1 & 85.2 $\pm$ 0.1 & - & - \\
                \midrule
                \multicolumn{7}{c}{\textbf{Dataset}: Office-Home} \\
                \midrule
                 & Art & Clipart & Product & Real & - & - \\
                \midrule
                \Algnameabbr{} (ResNet-18) & 61.4 $\pm$ 0.0 & 53.9 $\pm$ 0.2 & 76.0 $\pm$ 0.1 & 77.8 $\pm$ 0.0 & - & - \\
                \Algnameabbr{} (ResNet-50) & 71.3 $\pm$ 0.5 & 59.4 $\pm$ 0.4 & 80.3 $\pm$ 0.3 & 82.1 $\pm$ 0.0 & - & - \\
                \midrule
                \multicolumn{7}{c}{\textbf{Dataset}: VLCS} \\
                \midrule
                 & Caltech101 & LabelMe & SUN09 & VOC2007 & - & - \\
                \midrule
                \Algnameabbr{} (ResNet-18) & 97.2 $\pm$ 0.0 & 62.6 $\pm$ 0.1 & 75.0 $\pm$ 0.1 & 76.9 $\pm$ 0.1 & - & - \\
                \Algnameabbr{} (ResNet-50) & 98.3 $\pm$ 0.4 & 65.3 $\pm$ 0.1 & 77.2 $\pm$ 0.1 & 76.3 $\pm$ 0.1 & - & - \\
                \midrule
                \multicolumn{7}{c}{\textbf{Dataset}: Terra Incognita} \\
                \midrule
                 & Location 100 & Location 38 & Location 43 & Location 46 & - & - \\
                \midrule
                \Algnameabbr{} (ResNet-18) & 44.4 $\pm$ 0.4 & 49.6 $\pm$ 0.7 & 48.1 $\pm$ 0.2 & 37.3 $\pm$ 0.1 & - & - \\
                \Algnameabbr{} (ResNet-50) & 54.7 $\pm$ 0.3 & 48.1 $\pm$ 0.0 & 56.3 $\pm$ 0.3 & 42.3 $\pm$ 0.2 & - & -  \\
                \midrule
                \multicolumn{7}{c}{\textbf{Dataset}: DomainNet} \\
                \midrule
                 & Clipart & Infograph & Painting & Quickdraw & Real & Sketch \\
                \midrule
                \Algnameabbr{} (ResNet-18) & 44.6 $\pm$ 0.0 & 14.2 $\pm$ 0.0 & 39.7 $\pm$ 0.0 & 10.3 $\pm$ 0.0 & 45.6 $\pm$ 0.0 & 40.8 $\pm$ 0.0 \\
                \Algnameabbr{} (ResNet-50) & 63.4 $\pm$ 0.0 & 20.1 $\pm$ 0.1 & 48.2 $\pm$ 0.0 & 14.2 $\pm$ 0.0 & 63.4 $\pm$ 0.0 & 54.9 $\pm$ 0.0 \\
                \bottomrule
            \end{tabular}
        }
    }
    \label{tab:dispel_result}
\end{table*}

\subsection{Hyper-parameters of \Algnameabbr{}}
In the proposed \Algnameabbr{} framework, the hyper-parameters are composed of batch size, learning rate, and $\tau$ in Eq.~\ref{eq:mask}, where $\tau$ is the only hyper-parameter that is related to our algorithm. The hyper-parameters of \Algnameabbr{} for each benchmark dataset are shown in Tab.~\ref{tab:hyper-para}.

\subsection{Hyper-parameters of \Algnameabbr{} for Boosting Other Algorithms}
As shown in Sec.~\ref{sec:boosting}, we leverage our \Algnameabbr{} to further improve the prediction performance on unseen test domain for four existing domain generalization algorithms on PACS and Office Home, where all the DNN architectures are ResNet-50.
The hyper-parameters of the \Algnameabbr{} derivative models for the two datasets are shown in Tab.~\ref{tab:hyper-para-derivative}.

\section{Experimental Results of \Algnameabbr{}}
\label{appendix:exp_results}
To closely investigate the fine-grained behavior of \Algnameabbr{} in Sec.~\ref{sec:rq2}, we observe the prediction accuracy in each unseen test domain of all five domain generalization benchmark datasets.
In Tab.~\ref{tab:dispel_result}, we show the experimental results of \Algnameabbr{} on each unseen domain of five domain generalization benchmark datasets based on the two DNN architectures, ResNet-18 and ResNet-50.
Based on the experimental results on each unseen domain, we conclude the \textbf{Observation 2: \Algnameabbr{} possesses stable generalizing efficacy.}
The results show that \Algnameabbr{} maintains its stable efficacy in improving generalization ability over more different data distributions in more diverse classes of data. 
And these results reflect the purpose of the EMG module that considers each instance for fine-grained domain-specific feature masking.

\section{Visualization Analysis via t-SNE}
\label{appendix:t-sne}
To illustrate how \Algnameabbr{} improves generalization by blocking domain-specific features in the embedding space, we use t-SNE in the unseen test domains of all five benchmark datasets by comparing the embedding with and without \Algnameabbr{}, as shown in Fig.~\ref{fig:appendix_erm_embed_pacs} to Fig.~\ref{fig:appendix_dispel_embed_domainnet}.
The key observation is that \Algnameabbr{} aims to make each class more concentrated and separate them better.
Taking PACS as an example, by drawing down more precise decision boundaries, the predictor can achieve better accuracy in the unseen \textit{Art Painting} domain, in which \Algnameabbr{} enhances the most accuracy among the 4 domains as shown in Tab.~\ref{tab:tb3}.
Even in \textit{Cartoon} domain where \Algnameabbr{} only raises 0.7\% accuracy, it shows the same intention to concentrate the embedding distribution for each class in Fig.~\ref{fig:appendix_erm_embed_pacs}-(b) and Fig.~\ref{fig:appendix_dispel_embed_pacs}-(b).
As for the unseen \textit{Photo} domain, the base algorithm ERM has performed 96.7\% accuracy, which means that Fig.~\ref{fig:appendix_erm_embed_pacs}-(c) reveals what a high-quality representation looks like.
Compared to Fig.~\ref{fig:appendix_dispel_embed_pacs}-(c), \Algnameabbr{} follows the initial distribution and ameliorates the embedding to compress the distributions of each class.

Based on the t-SNE visualization analysis, we conclude \textbf{Observation 3: \Algnameabbr{} concentrate the distribution of each class embedding.}
The t-SNE analysis demonstrates the superiority of \Algnameabbr{}, which improves the domain generalization ability of the fine-tuned ERM by concentrating the distribution of embeddings in the same class.

\begin{figure*}[tbh!]
    \centering
    \subfigure[Art Painting]{
    \centering
    \begin{minipage}[t]{0.23\linewidth}
	    \includegraphics[width=0.99\linewidth]{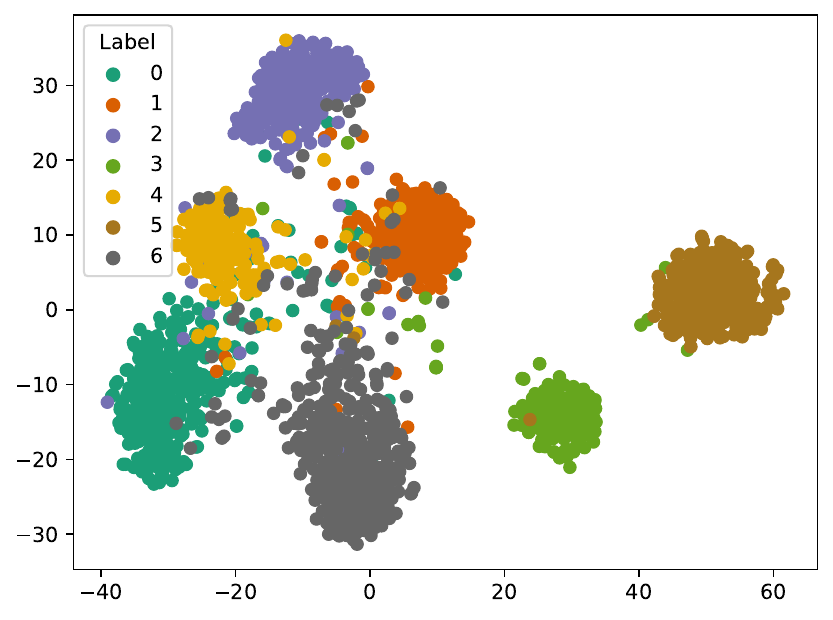}
    \end{minipage}%
    }
    \subfigure[Cartoon]{
    \centering
    \begin{minipage}[t]{0.23\linewidth}
	    \includegraphics[width=0.99\linewidth]{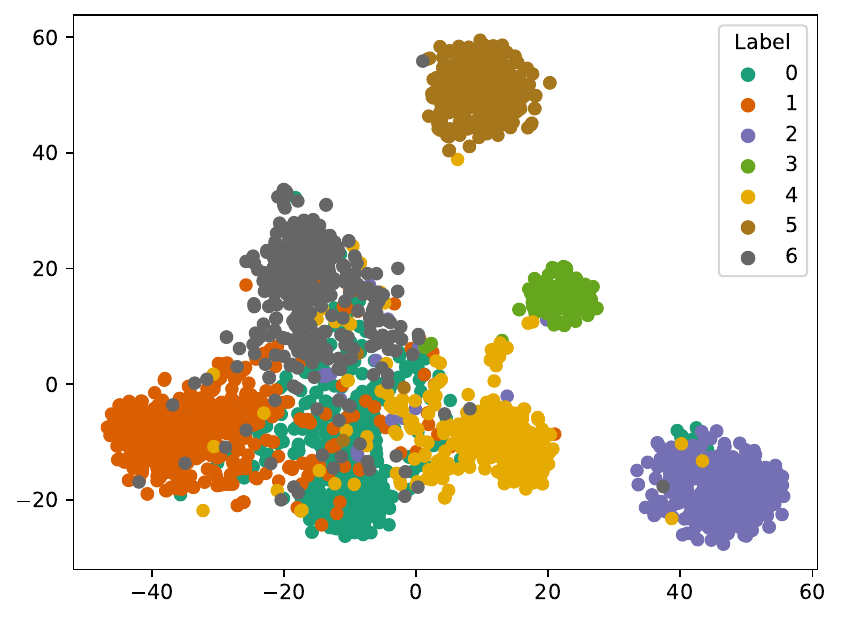}
    \end{minipage}%
    }
    \subfigure[Photo]{
    \centering
    \begin{minipage}[t]{0.23\linewidth}
	    \includegraphics[width=0.99\linewidth]{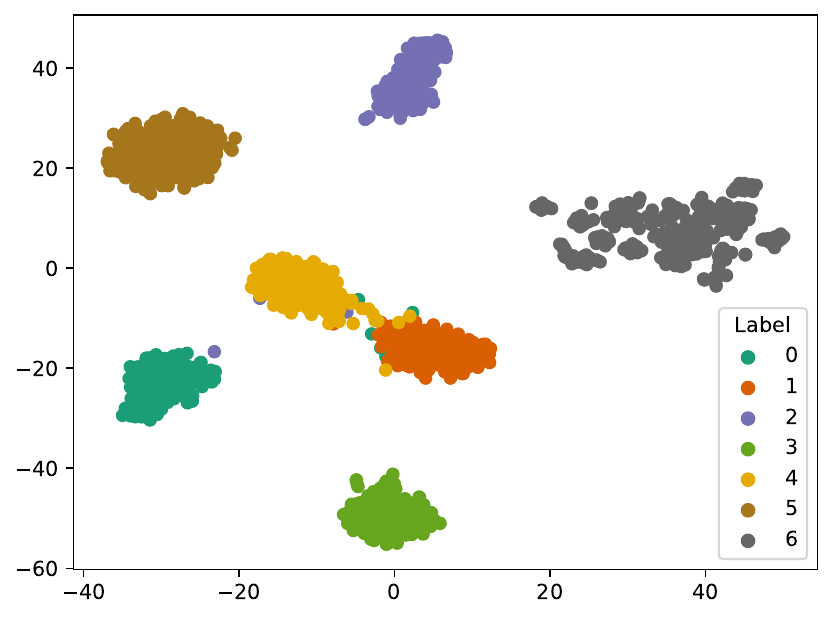}
    \end{minipage}%
    }
    \subfigure[Sketch]{
    \centering
    \begin{minipage}[t]{0.23\linewidth}
	    \includegraphics[width=0.99\linewidth]{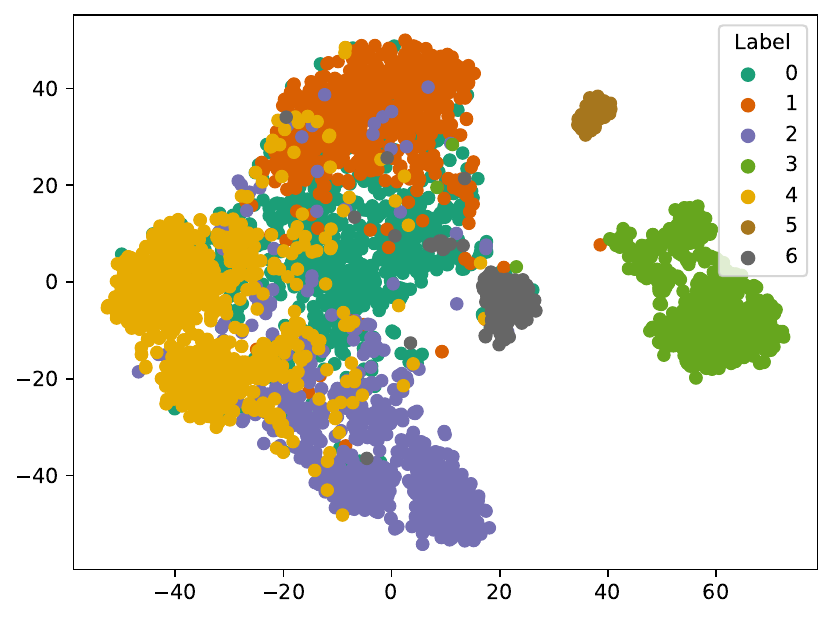}
    \end{minipage}%
    }
    \vspace{-4mm}
    \caption{t-SNE visualization of ERM embedding in four unseen test domains of PACS.}
    \label{fig:appendix_erm_embed_pacs}
    \vspace{-3mm}
\end{figure*}

\begin{figure*}[tbh!]
\setlength{\abovecaptionskip}{0mm}
\setlength{\belowcaptionskip}{-5mm}
    \centering
    \subfigure[Art Painting]{
    \centering
    \begin{minipage}[t]{0.23\linewidth}
	\includegraphics[width=1.0\linewidth]{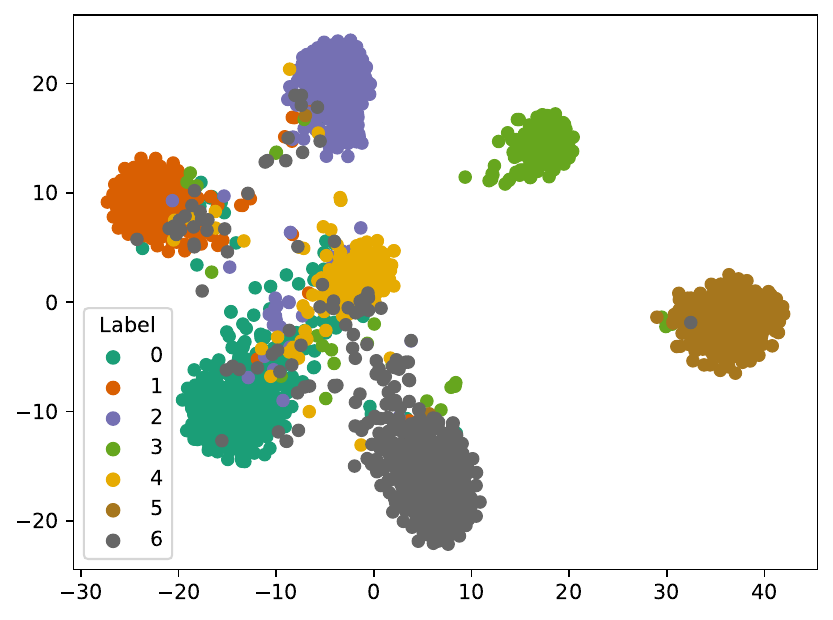}
    \end{minipage}%
    }
    \subfigure[Cartoon]{
    \centering
    \begin{minipage}[t]{0.23\linewidth}
	\includegraphics[width=1.0\linewidth]{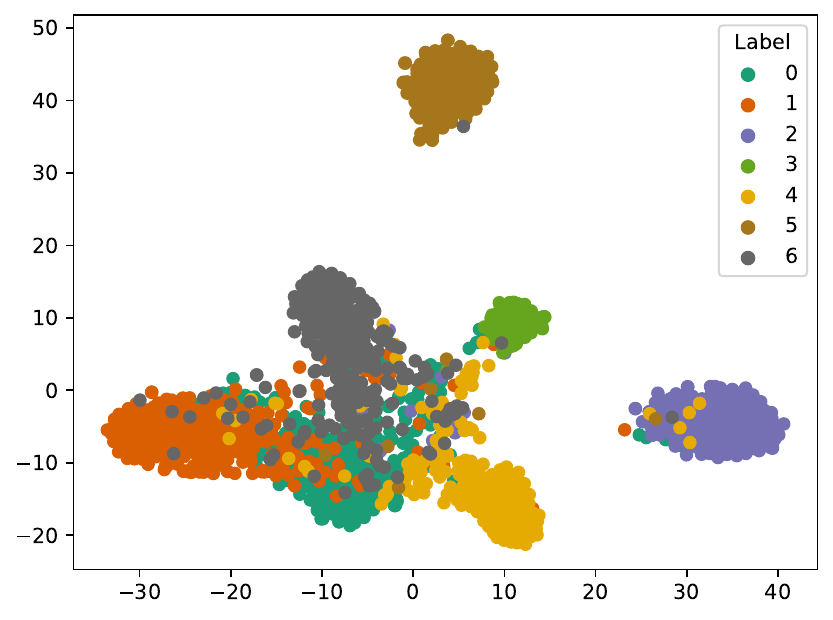}
    \end{minipage}%
    }
    \subfigure[Photo]{
    \centering
    \begin{minipage}[t]{0.23\linewidth}
	\includegraphics[width=1.0\linewidth]{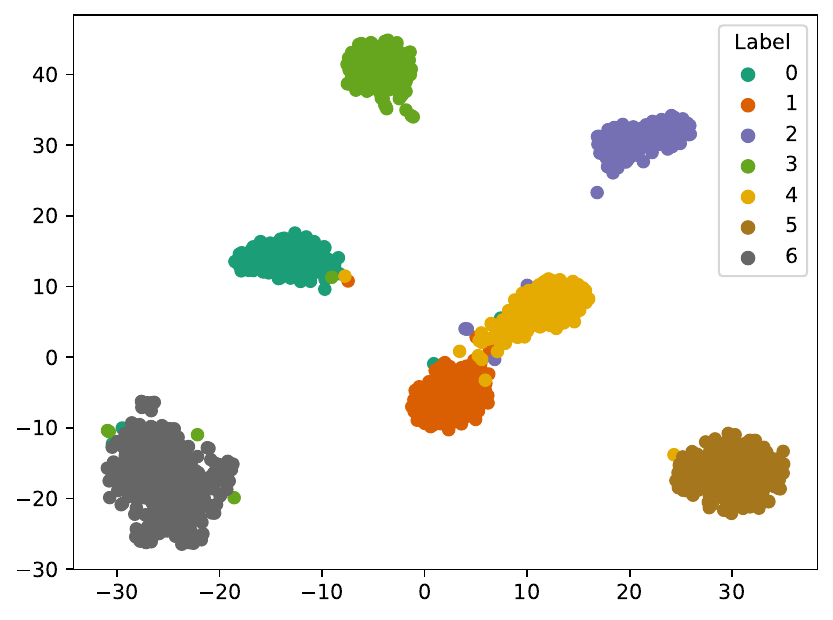}
    \end{minipage}%
    }
    \subfigure[Sketch]{
    \centering
    \begin{minipage}[t]{0.23\linewidth}
	\includegraphics[width=1.0\linewidth]{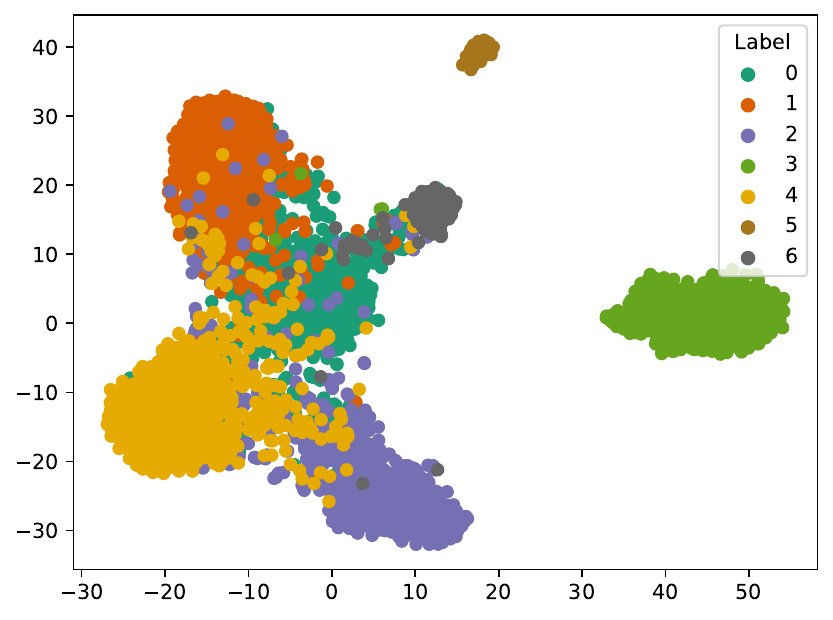}
    \end{minipage}%
    }
    \caption{t-SNE visualization of \Algnameabbr{} embedding in four unseen test domains of PACS.}
    \vspace{1mm}
    \label{fig:appendix_dispel_embed_pacs}
\end{figure*}

\begin{figure*}[tbh!]
    \centering
    \subfigure[Art]{
    \centering
    \begin{minipage}[t]{0.23\linewidth}
	    \includegraphics[width=1.0\linewidth]{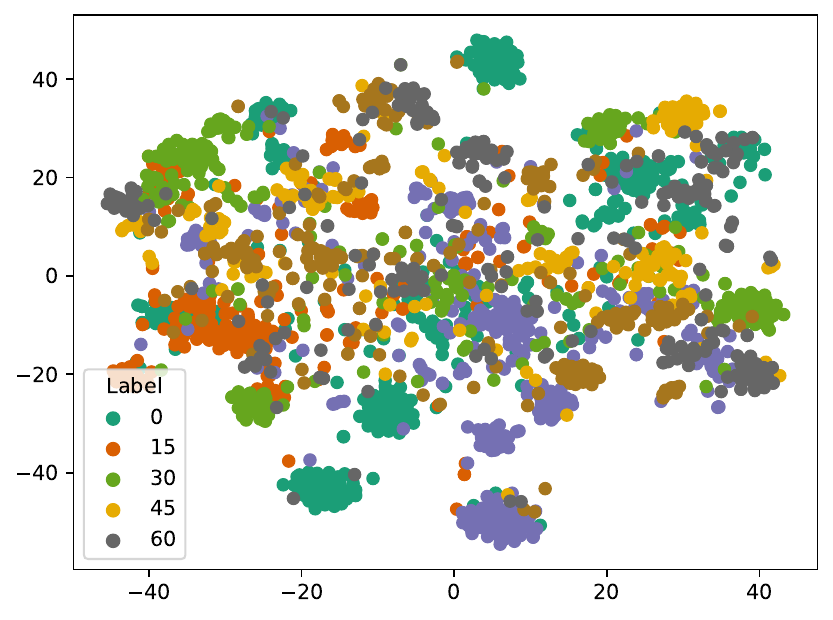}
    \end{minipage}%
    }
    \subfigure[Clipart]{
    \centering
    \begin{minipage}[t]{0.23\linewidth}
	    \includegraphics[width=0.99\linewidth]{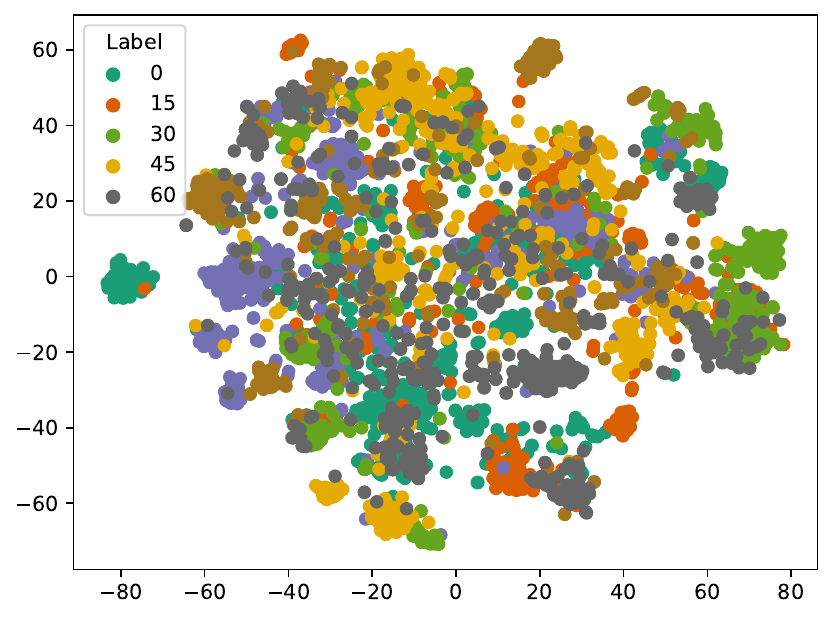}
    \end{minipage}%
    }
    \subfigure[Product]{
    \centering
    \begin{minipage}[t]{0.23\linewidth}
	    \includegraphics[width=0.99\linewidth]{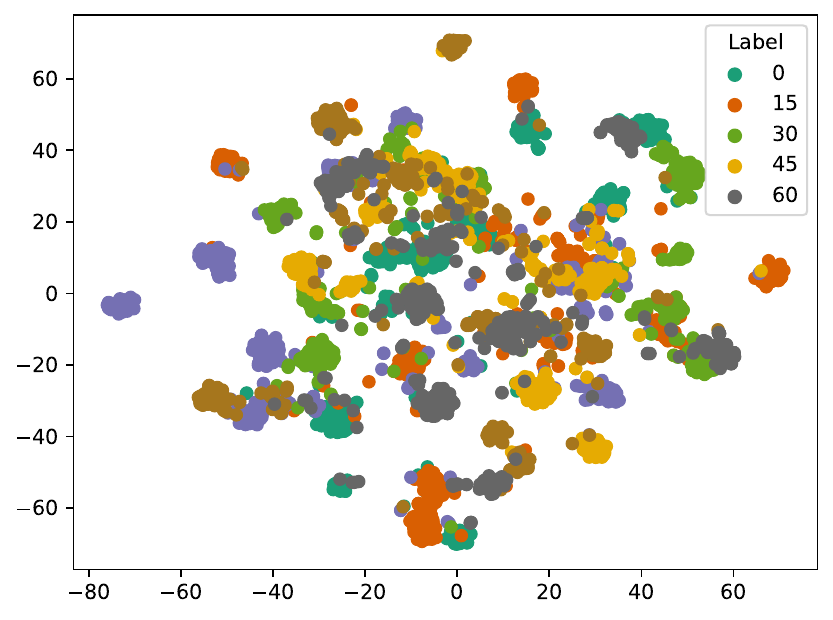}
    \end{minipage}%
    }
    \subfigure[Real]{
    \centering
    \begin{minipage}[t]{0.23\linewidth}
	    \includegraphics[width=0.99\linewidth]{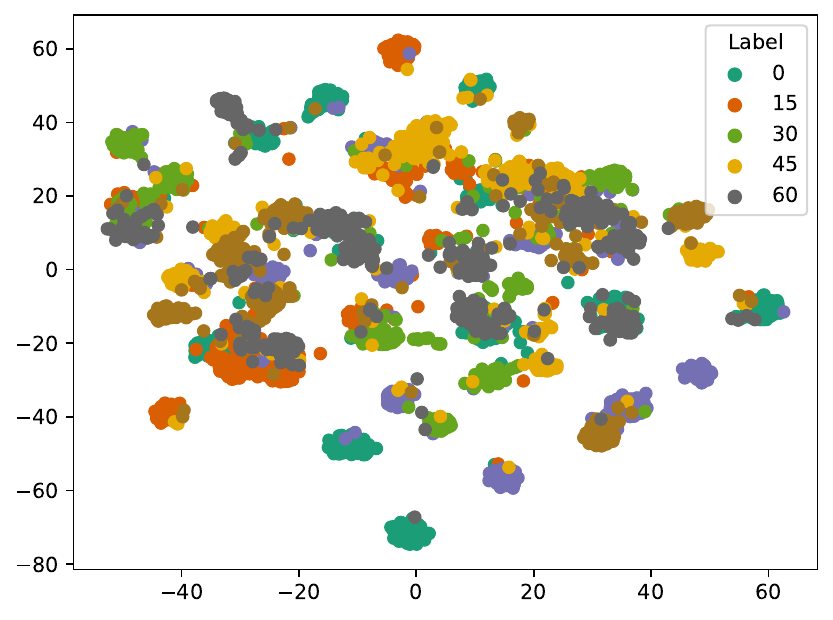}
    \end{minipage}%
    }
    \vspace{-4mm}
    \caption{t-SNE visualization of ERM embedding in four unseen test domains of Office-Home.}
    \label{fig:appendix_erm_embed_office}
    \vspace{-3mm}
\end{figure*}

\begin{figure*}[tbh!]
\setlength{\abovecaptionskip}{0mm}
\setlength{\belowcaptionskip}{-5mm}
    \centering
    \subfigure[Art]{
    \centering
    \begin{minipage}[t]{0.23\linewidth}
	\includegraphics[width=1.0\linewidth]{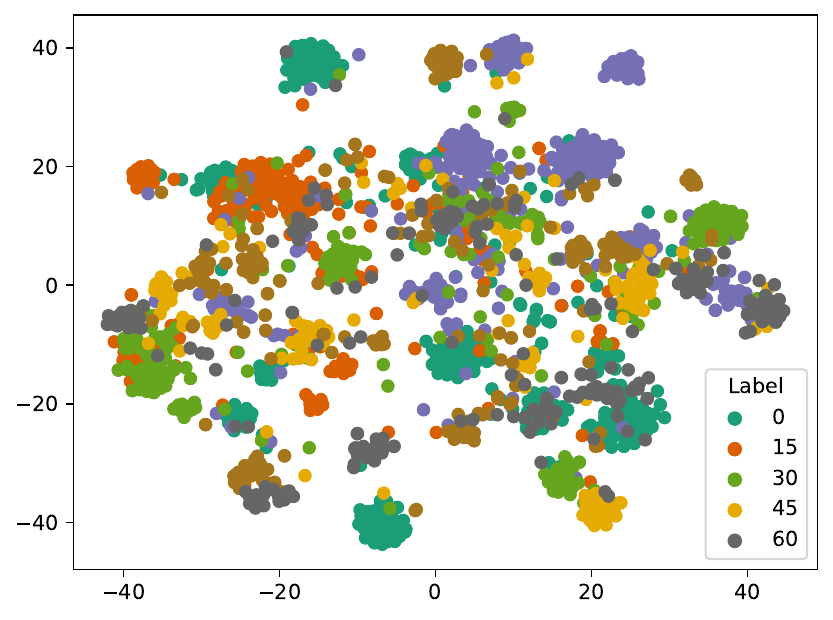}
    \end{minipage}%
    }
    \subfigure[Clipart]{
    \centering
    \begin{minipage}[t]{0.23\linewidth}
	\includegraphics[width=1.0\linewidth]{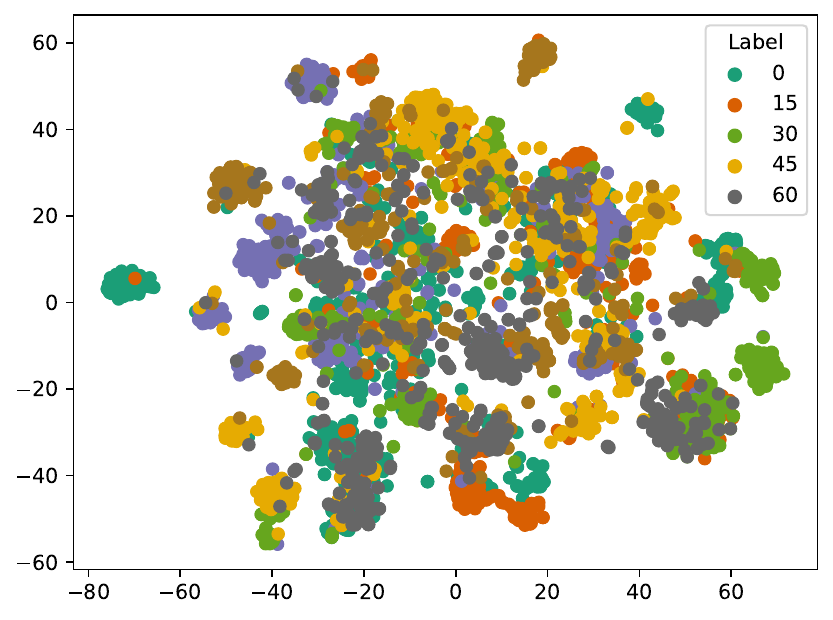}
    \end{minipage}%
    }
    \subfigure[Product]{
    \centering
    \begin{minipage}[t]{0.23\linewidth}
	\includegraphics[width=1.0\linewidth]{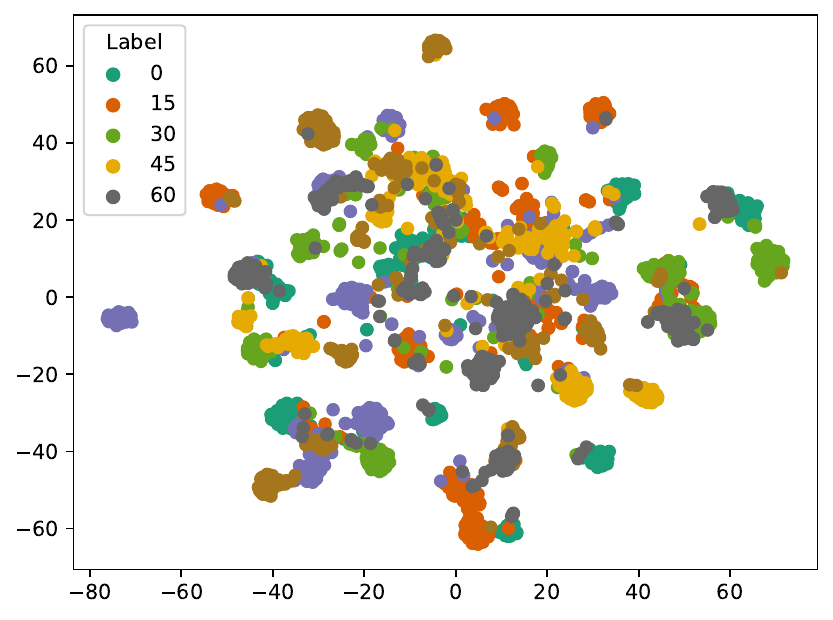}
    \end{minipage}%
    }
    \subfigure[Real]{
    \centering
    \begin{minipage}[t]{0.23\linewidth}
	\includegraphics[width=1.0\linewidth]{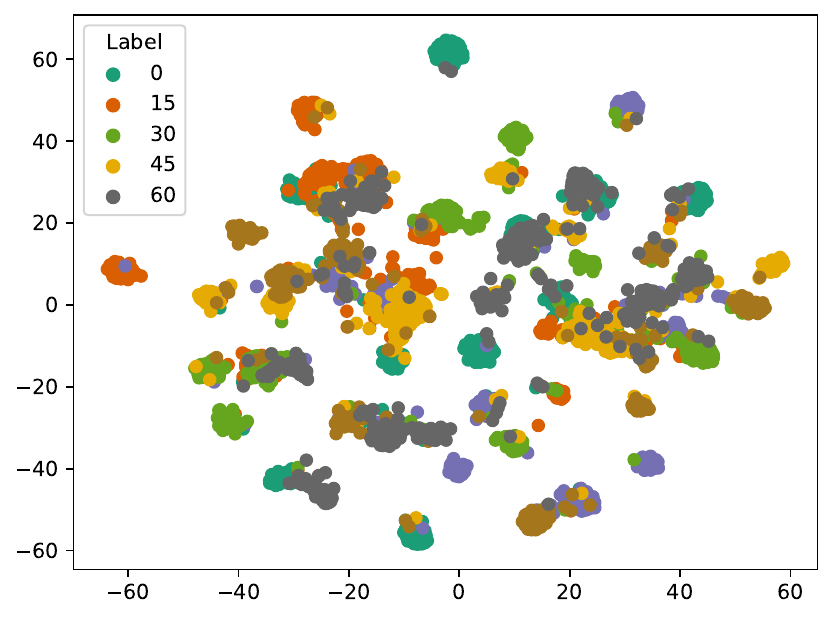}
    \end{minipage}%
    }
    \caption{t-SNE visualization of \Algnameabbr{} embedding in four unseen test domains of Office-Home.}
    \vspace{1mm}
    \label{fig:appendix_dispel_embed_office}
\end{figure*}

\begin{figure*}[tbh!]
    \centering
    \subfigure[Caltech101]{
    \centering
    \begin{minipage}[t]{0.23\linewidth}
	    \includegraphics[width=1.0\linewidth]{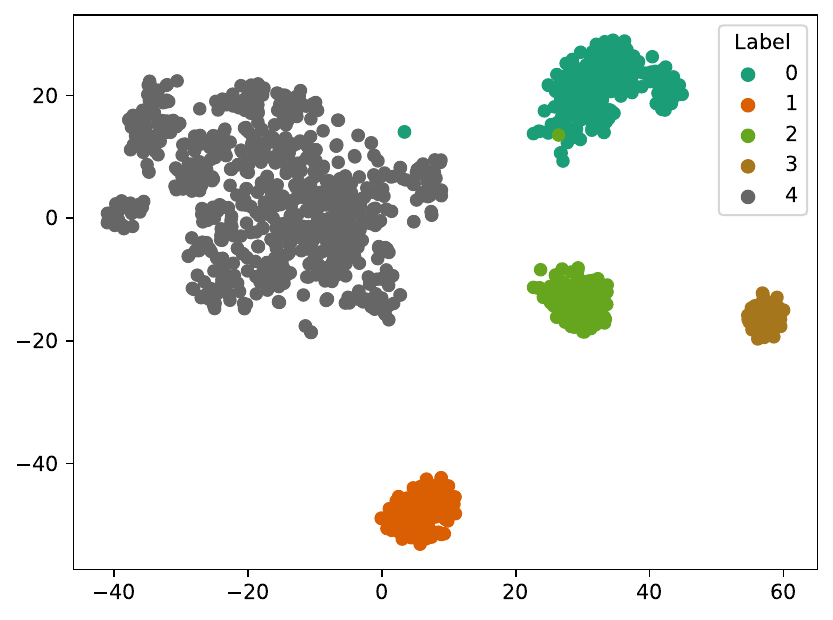}
    \end{minipage}%
    }
    \subfigure[LabelMe]{
    \centering
    \begin{minipage}[t]{0.23\linewidth}
	    \includegraphics[width=0.99\linewidth]{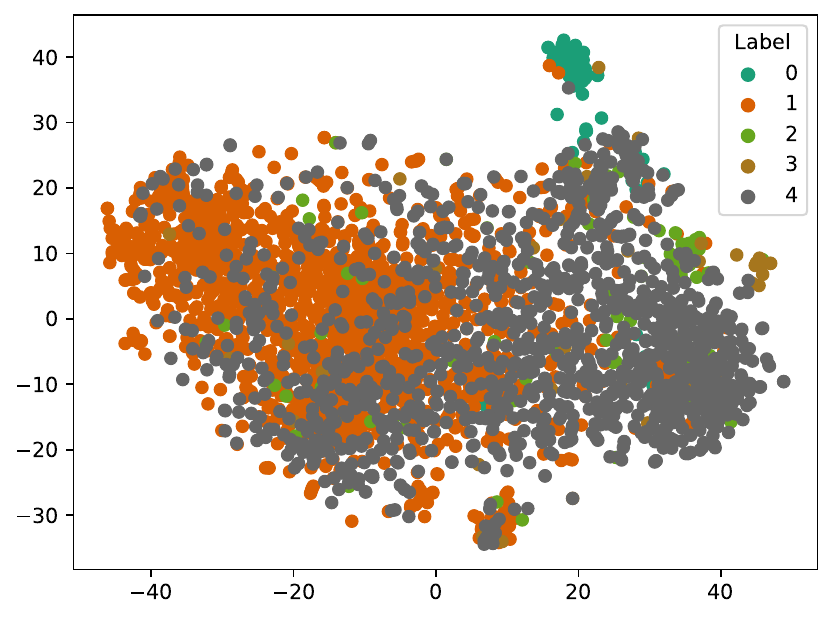}
    \end{minipage}%
    }
    \subfigure[SUN09]{
    \centering
    \begin{minipage}[t]{0.23\linewidth}
	    \includegraphics[width=0.99\linewidth]{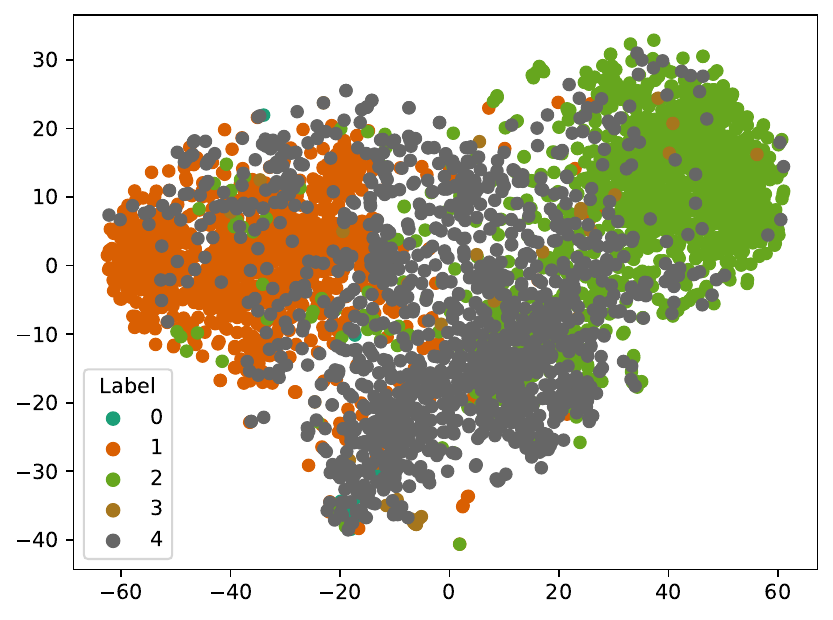}
    \end{minipage}%
    }
    \subfigure[VOC2007]{
    \centering
    \begin{minipage}[t]{0.23\linewidth}
	    \includegraphics[width=0.99\linewidth]{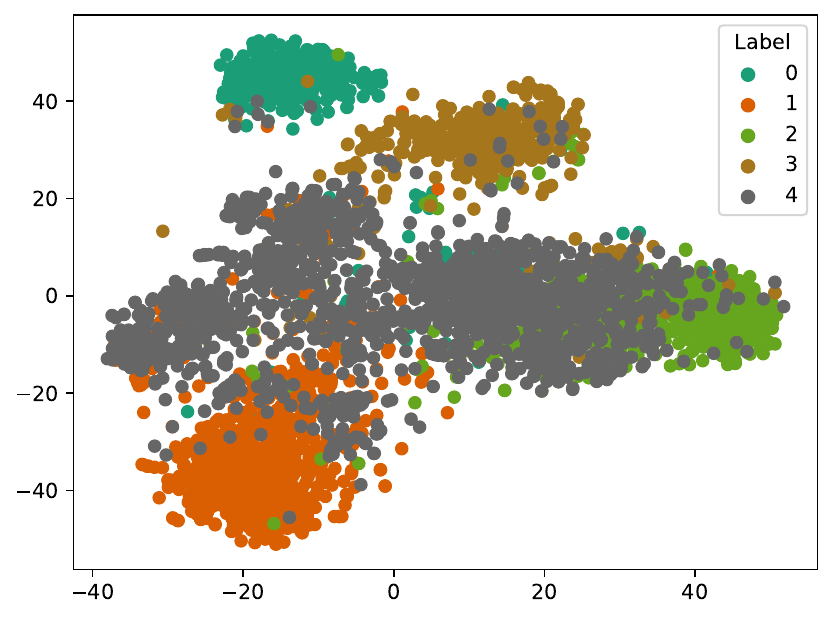}
    \end{minipage}%
    }
    \vspace{-4mm}
    \caption{t-SNE visualization of ERM embedding in four unseen test domains of VLCS.}
    \label{fig:appendix_erm_embed_vlcs}
    \vspace{-3mm}
\end{figure*}

\begin{figure*}[tbh!]
\setlength{\abovecaptionskip}{0mm}
\setlength{\belowcaptionskip}{-5mm}
    \centering
    \subfigure[Caltech101]{
    \centering
    \begin{minipage}[t]{0.23\linewidth}
	\includegraphics[width=1.0\linewidth]{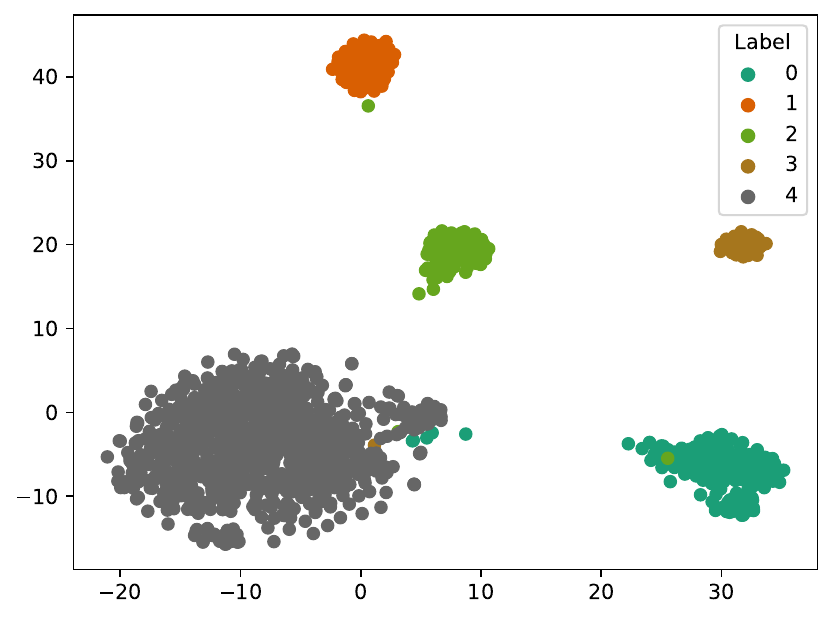}
    \end{minipage}%
    }
    \subfigure[LabelMe]{
    \centering
    \begin{minipage}[t]{0.23\linewidth}
	\includegraphics[width=1.0\linewidth]{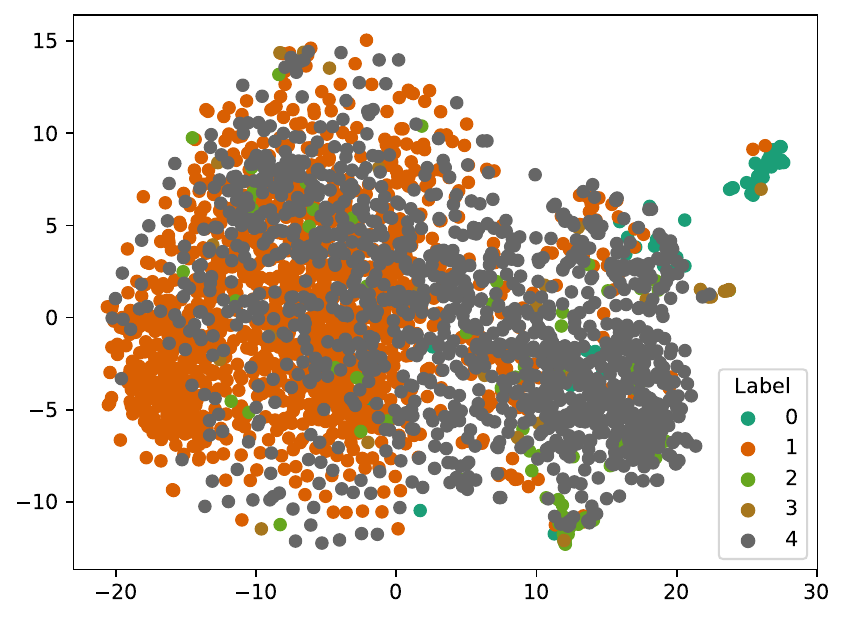}
    \end{minipage}%
    }
    \subfigure[SUN09]{
    \centering
    \begin{minipage}[t]{0.23\linewidth}
	\includegraphics[width=1.0\linewidth]{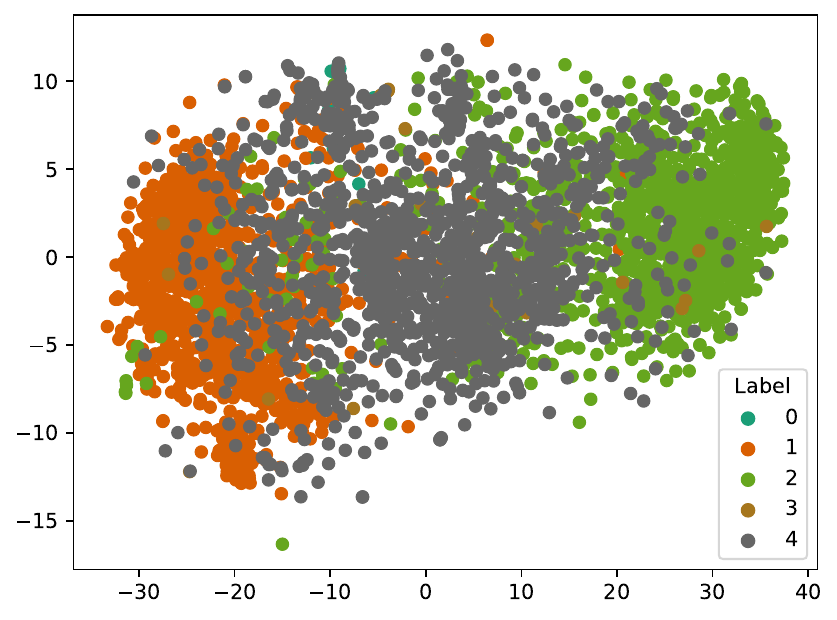}
    \end{minipage}%
    }
    \subfigure[VOC2007]{
    \centering
    \begin{minipage}[t]{0.23\linewidth}
	\includegraphics[width=1.0\linewidth]{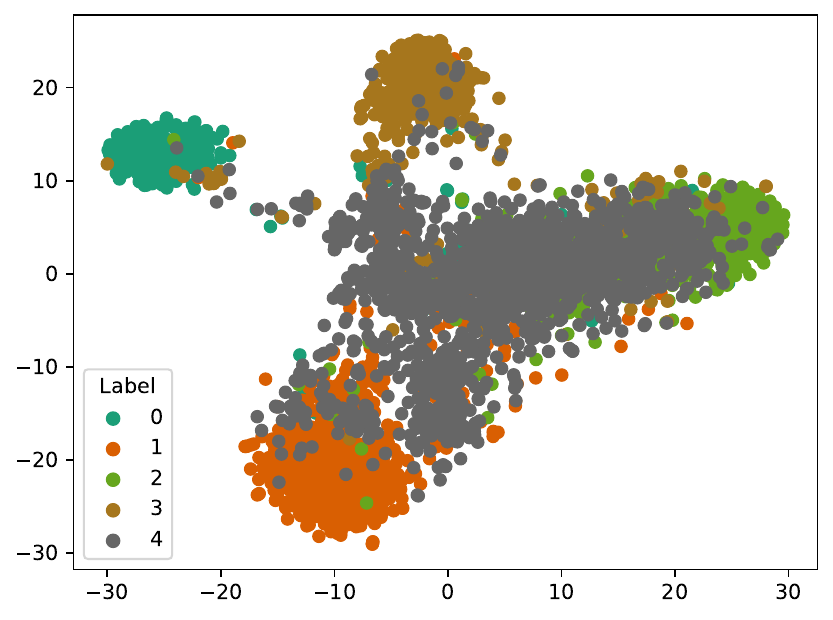}
    \end{minipage}%
    }
    \caption{t-SNE visualization of \Algnameabbr{} embedding in four unseen test domains of VLCS.}
    \vspace{1mm}
    \label{fig:appendix_dispel_embed_vlcs}
\end{figure*}

\begin{figure*}[tbh!]
    \centering
    \subfigure[Location 100]{
    \centering
    \begin{minipage}[t]{0.23\linewidth}
	    \includegraphics[width=1.0\linewidth]{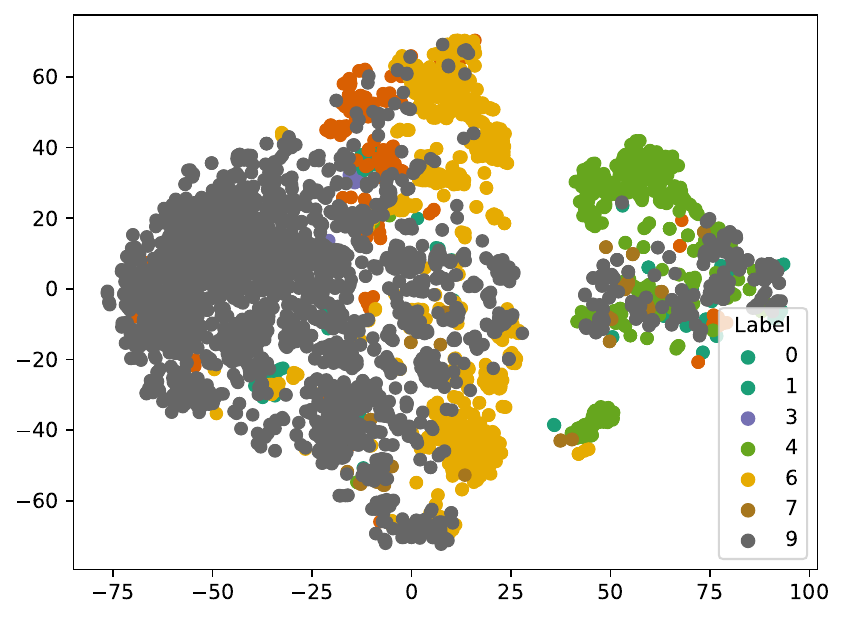}
    \end{minipage}%
    }
    \subfigure[Location 38]{
    \centering
    \begin{minipage}[t]{0.23\linewidth}
	    \includegraphics[width=0.99\linewidth]{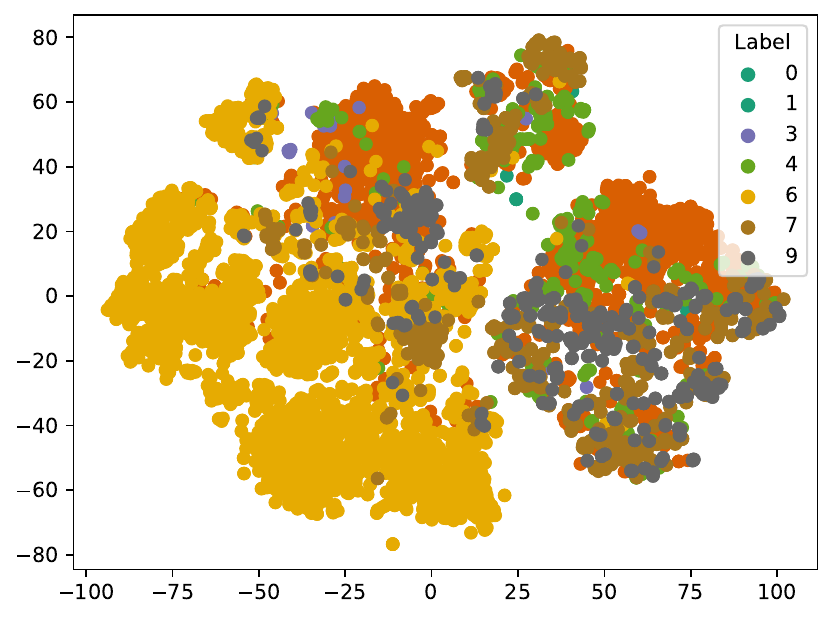}
    \end{minipage}%
    }
    \subfigure[Location 43]{
    \centering
    \begin{minipage}[t]{0.23\linewidth}
	    \includegraphics[width=0.99\linewidth]{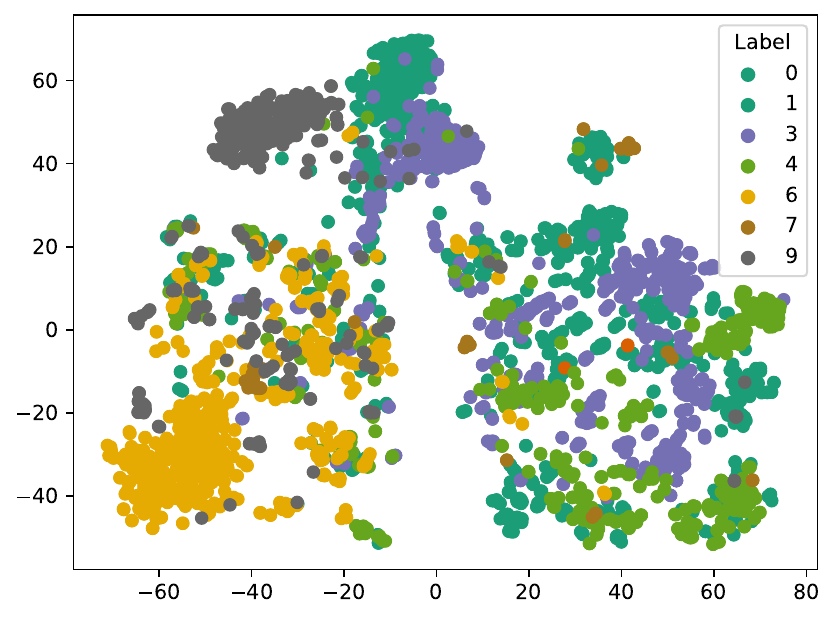}
    \end{minipage}%
    }
    \subfigure[Location 46]{
    \centering
    \begin{minipage}[t]{0.23\linewidth}
	    \includegraphics[width=0.99\linewidth]{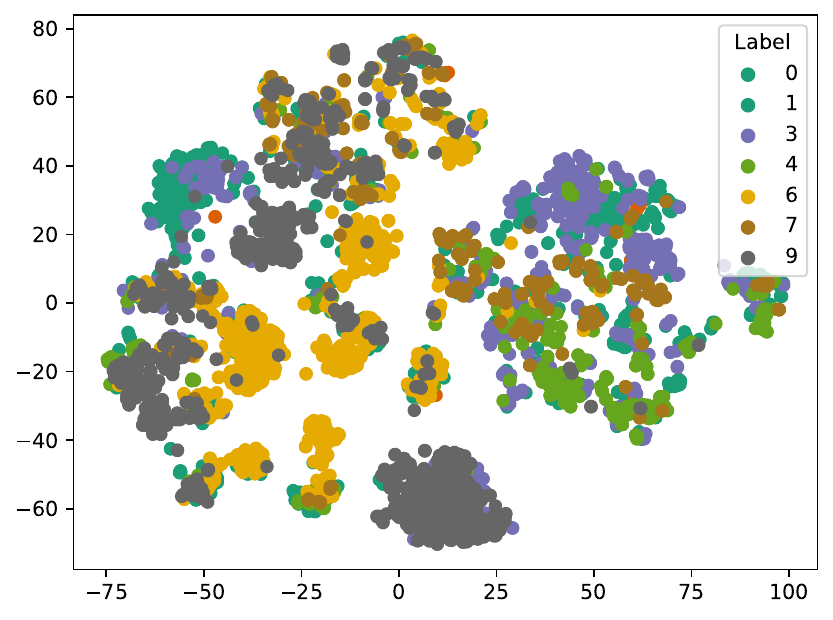}
    \end{minipage}%
    }
    \vspace{-4mm}
    \caption{t-SNE visualization of ERM embedding in four unseen test domains of Terra Incognita.}
    \label{fig:appendix_erm_embed_terra}
    \vspace{-3mm}
\end{figure*}

\begin{figure*}[tbh!]
\setlength{\abovecaptionskip}{0mm}
\setlength{\belowcaptionskip}{-5mm}
    \centering
    \subfigure[Location 100]{
    \centering
    \begin{minipage}[t]{0.23\linewidth}
	\includegraphics[width=1.0\linewidth]{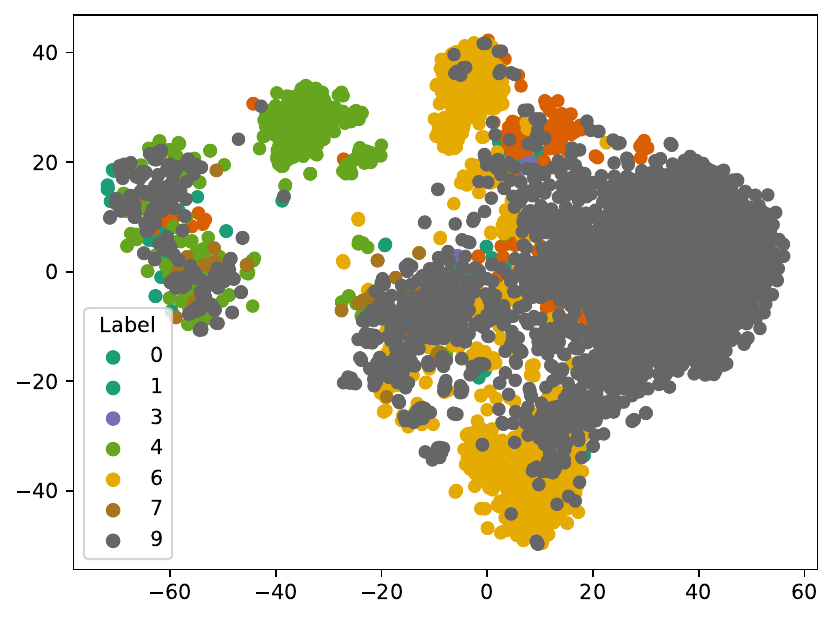}
    \end{minipage}%
    }
    \subfigure[Location 38]{
    \centering
    \begin{minipage}[t]{0.23\linewidth}
	\includegraphics[width=1.0\linewidth]{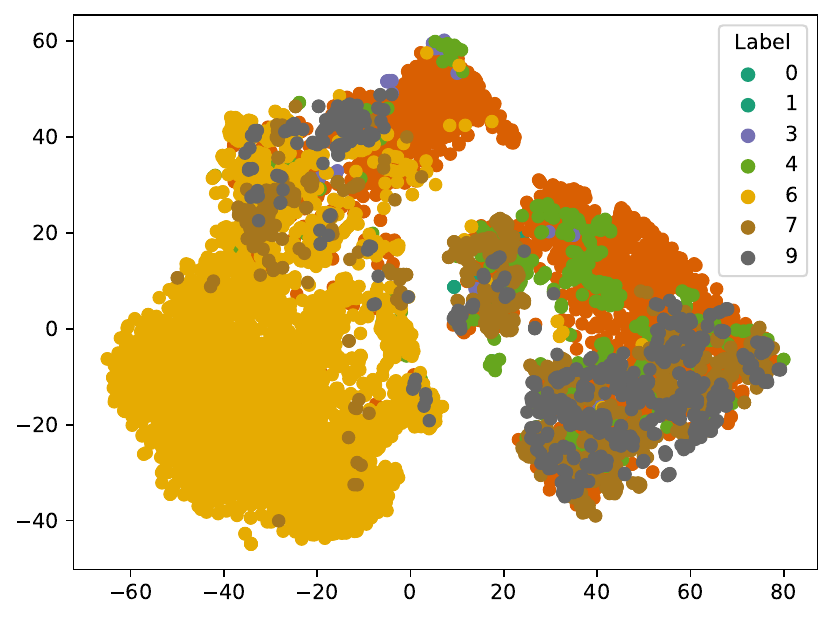}
    \end{minipage}%
    }
    \subfigure[Location 43]{
    \centering
    \begin{minipage}[t]{0.23\linewidth}
	\includegraphics[width=1.0\linewidth]{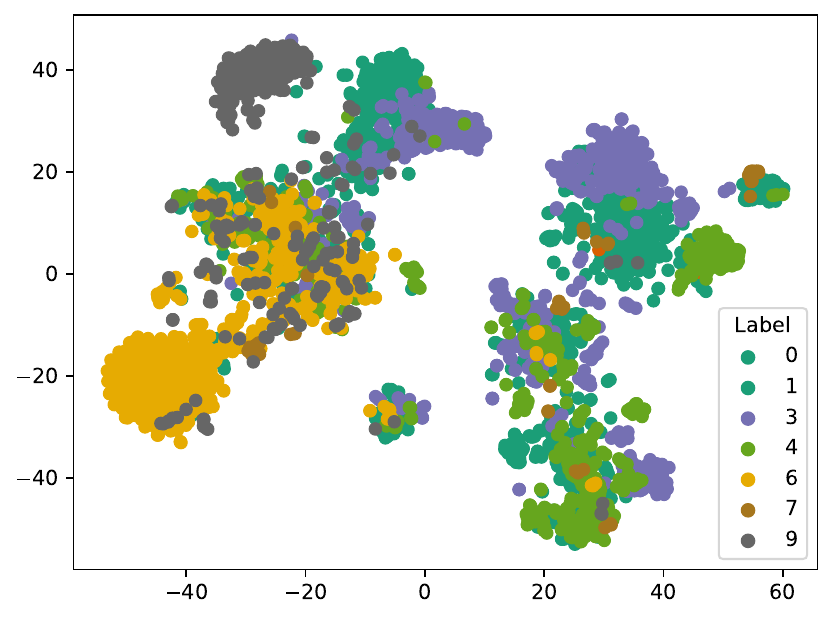}
    \end{minipage}%
    }
    \subfigure[Location 46]{
    \centering
    \begin{minipage}[t]{0.23\linewidth}
	\includegraphics[width=1.0\linewidth]{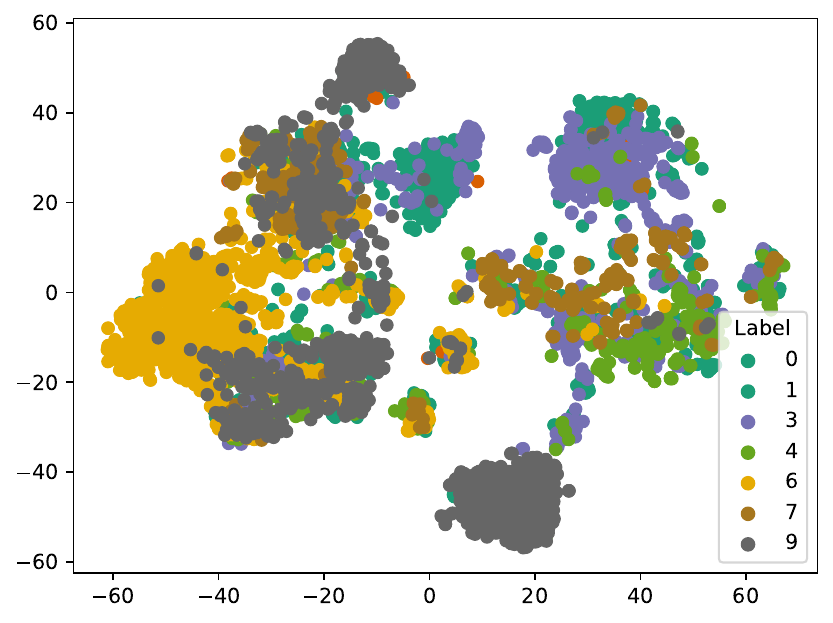}
    \end{minipage}%
    }
    \caption{t-SNE visualization of \Algnameabbr{} embedding in four unseen test domains of Terra Incognita.}
    \vspace{1mm}
    \label{fig:appendix_dispel_embed_terra}
\end{figure*}

\begin{figure*}[tbh!]
    \centering
    \subfigure[Clipart]{
    \centering
    \begin{minipage}[t]{0.153\linewidth}
	    \includegraphics[width=1.0\linewidth]{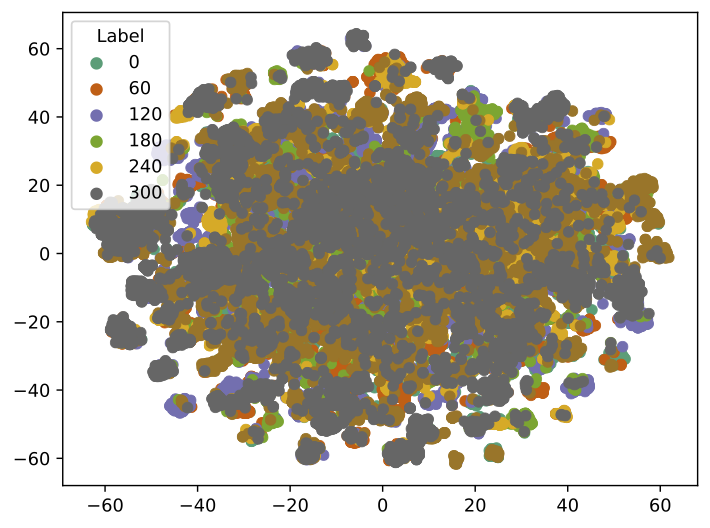}
    \end{minipage}%
    }
    \subfigure[Infograph]{
    \centering
    \begin{minipage}[t]{0.153\linewidth}
	    \includegraphics[width=0.99\linewidth]{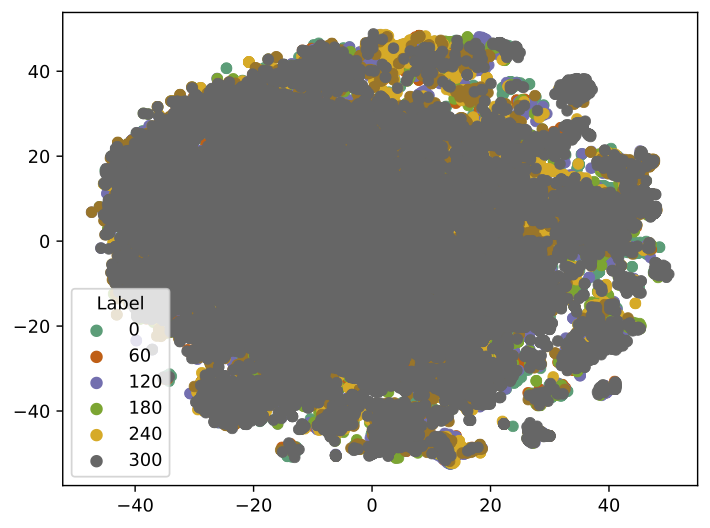}
    \end{minipage}%
    }
    \subfigure[Painting]{
    \centering
    \begin{minipage}[t]{0.153\linewidth}
	    \includegraphics[width=0.99\linewidth]{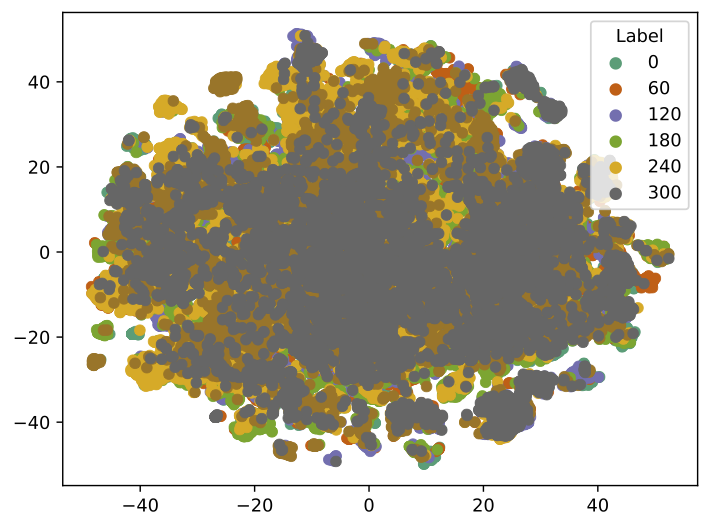}
    \end{minipage}%
    }
    \subfigure[Quickdraw]{
    \centering
    \begin{minipage}[t]{0.153\linewidth}
	    \includegraphics[width=0.99\linewidth]{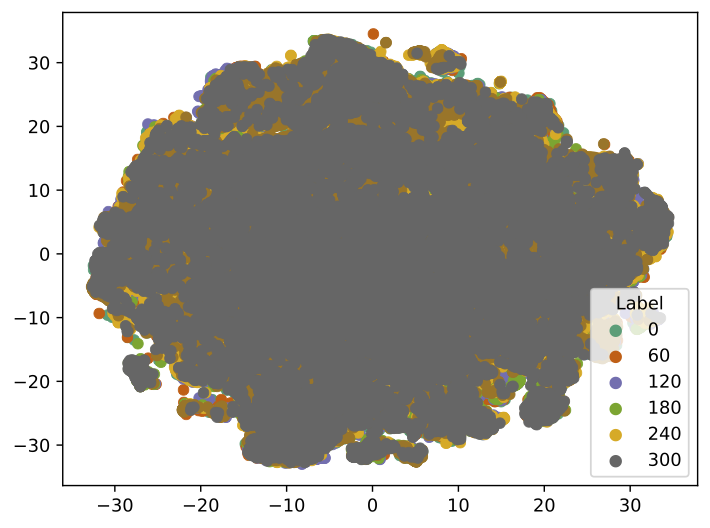}
    \end{minipage}%
    }
    \subfigure[Real]{
    \centering
    \begin{minipage}[t]{0.153\linewidth}
	    \includegraphics[width=1.0\linewidth]{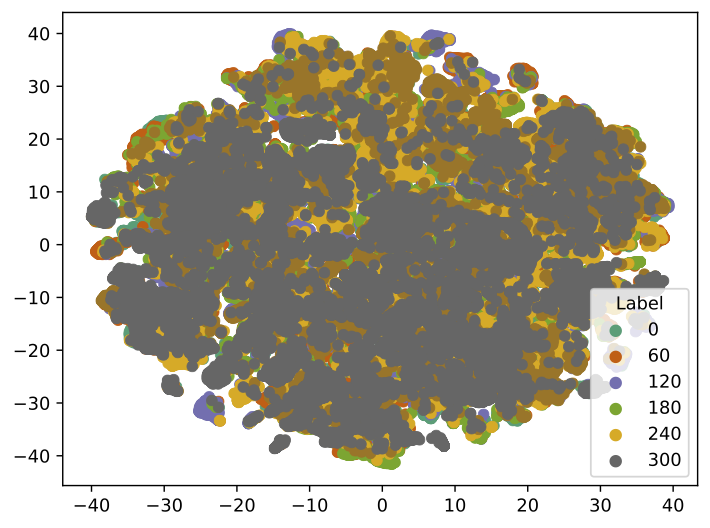}
    \end{minipage}%
    }
    \subfigure[Sketch]{
    \centering
    \begin{minipage}[t]{0.153\linewidth}
	    \includegraphics[width=1.0\linewidth]{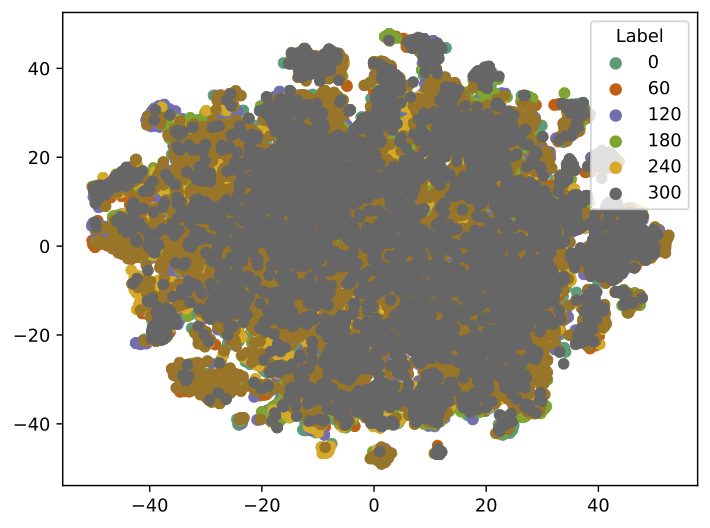}
    \end{minipage}%
    }
    \vspace{-4mm}
    \caption{t-SNE visualization of ERM embedding in six unseen test domains of DomainNet.}
    \label{fig:appendix_erm_embed_domainnet}
    \vspace{-3mm}
\end{figure*}

\begin{figure*}[tbh!]
\setlength{\abovecaptionskip}{0mm}
\setlength{\belowcaptionskip}{-5mm}
    \centering
    \subfigure[Clipart]{
    \centering
    \begin{minipage}[t]{0.153\linewidth}
	\includegraphics[width=1.0\linewidth]{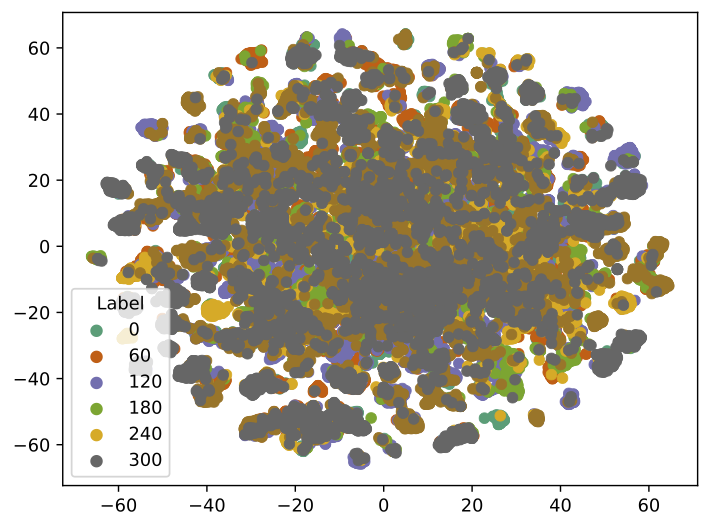}
    \end{minipage}%
    }
    \subfigure[Infograph]{
    \centering
    \begin{minipage}[t]{0.153\linewidth}
	\includegraphics[width=1.0\linewidth]{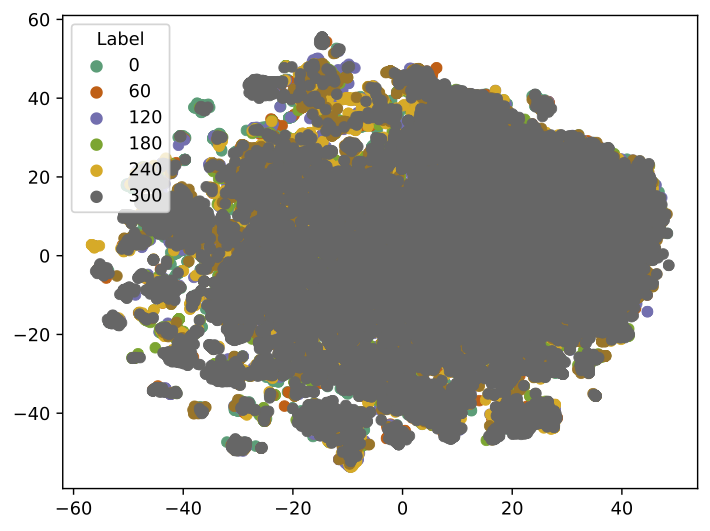}
    \end{minipage}%
    }
    \subfigure[Painting]{
    \centering
    \begin{minipage}[t]{0.153\linewidth}
	\includegraphics[width=1.0\linewidth]{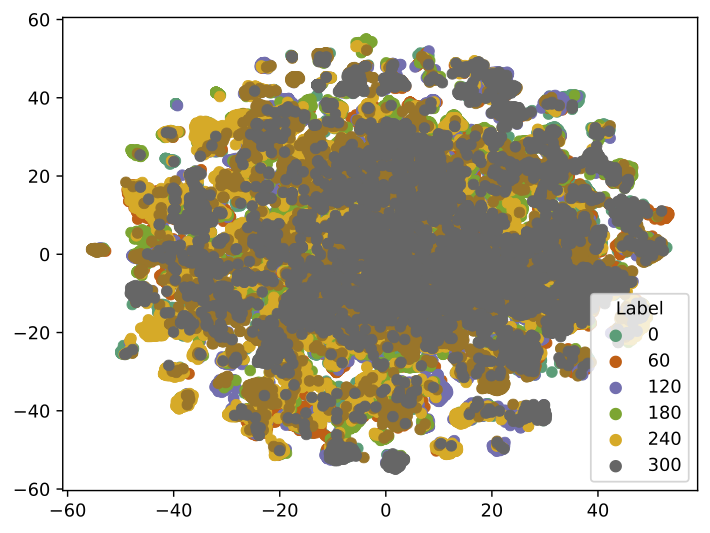}
    \end{minipage}%
    }
    \subfigure[Quickdraw]{
    \centering
    \begin{minipage}[t]{0.153\linewidth}
	\includegraphics[width=1.0\linewidth]{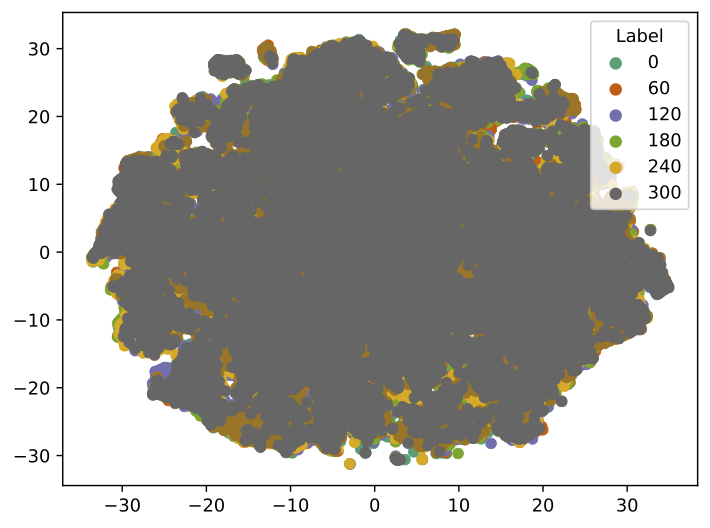}
    \end{minipage}%
    }
    \subfigure[Real]{
    \centering
    \begin{minipage}[t]{0.153\linewidth}
	\includegraphics[width=1.0\linewidth]{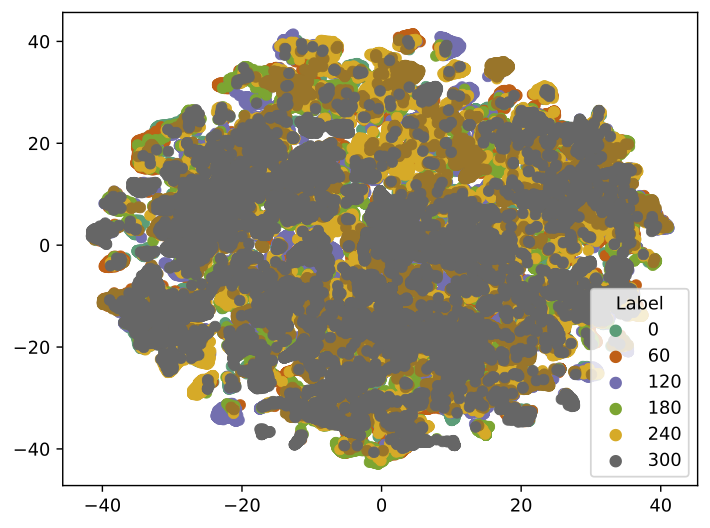}
    \end{minipage}%
    }
    \subfigure[Sketch]{
    \centering
    \begin{minipage}[t]{0.153\linewidth}
	\includegraphics[width=1.0\linewidth]{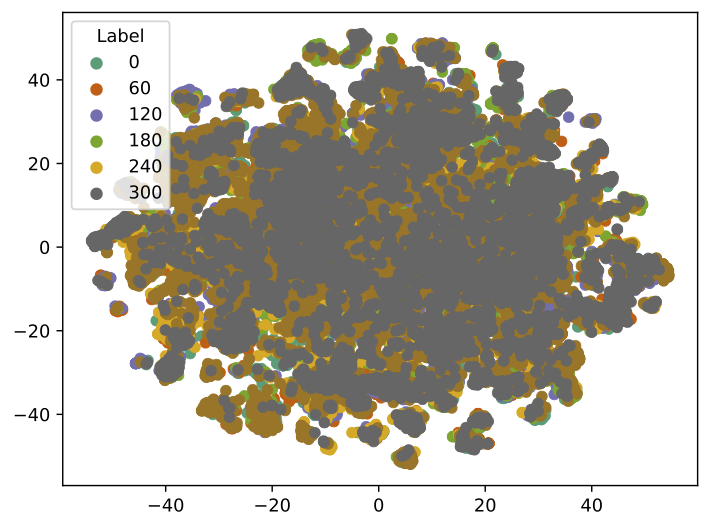}
    \end{minipage}%
    }
    \caption{t-SNE visualization of \Algnameabbr{} embedding in six unseen test domains of DomainNet.}
    \vspace{1mm}
    \label{fig:appendix_dispel_embed_domainnet}
\end{figure*}

\end{document}